%% file: aaai_pre_sgld.tex
\def\year{2016}
\documentclass[letterpaper]{article}
\include{subtex/mlVecMat}

\newcommand{\RN}[1]{%
  \textup{\lowercase\expandafter{\it \romannumeral#1}}%
}

\usepackage{aaai16}
\usepackage{times}
\usepackage{helvet}
\usepackage{courier}

\usepackage{marvosym}
\newcommand{\envelope}{\raisebox{-.5pt}{\scalebox{1.05}{\Letter}}\kern1.0pt}

\begin{document}
%
\title{Preconditioned Stochastic Gradient Langevin Dynamics for \\Deep Neural Networks
	} 
\author{Chunyuan Li$^{1}$, Changyou Chen$^{1}$, David Carlson$^{2}$ and Lawrence Carin$^{1}$\\
	$^{1}$Department of Electrical and Computer Engineering, Duke University\\
	$^{2}$Department of Statistics and Grossman Center, Columbia University\\
	{\footnotesize
	\href{mailto:chunyuan.li@duke.edu}{  \;\;\;\;\nolinkurl{chunyuan.li@duke.edu,}  } 
	\href{mailto:cchangyou@gmail.com}{\nolinkurl{cchangyou@gmail.com,}  } 
	\href{mailto:david.edwin.carlson@gmail.com}{\nolinkurl{david.edwin.carlson@gmail.com,}  } 
	\href{mailto:lcarin@duke.edu}{\nolinkurl{lcarin@duke.edu} }
	}
}
\maketitle
\begin{abstract}
\begin{quote}
Effective training of deep neural networks suffers from two main issues.  The first is that the parameter spaces of these models exhibit pathological curvature. Recent methods address this problem by using adaptive preconditioning for Stochastic Gradient Descent (SGD).  These methods improve convergence by adapting to the local geometry of parameter space. A second issue is overfitting, which is typically addressed by early stopping.  However, recent work has demonstrated that Bayesian model averaging mitigates this problem.  The posterior can be sampled by using Stochastic Gradient Langevin Dynamics (SGLD). However, the rapidly changing curvature renders default SGLD methods inefficient.  Here, we propose combining adaptive preconditioners with SGLD.  In support of this idea, we give theoretical properties on asymptotic convergence and predictive risk. We also provide empirical results for Logistic Regression, Feedforward Neural Nets, and Convolutional Neural Nets, demonstrating that our preconditioned SGLD method gives state-of-the-art performance on these models.

\end{quote}
\end{abstract}

\section{Introduction}
Deep Neural Networks (DNNs) have recently generated significant interest, largely due to their state-of-the-art performance on a wide variety of tasks, such as image classification~\cite{krizhevsky2012imagenet} and language modeling~\cite{sutskever2014sequence}.  Despite this significant empirical success, it remains a challenge to effectively train DNNs.  This is due to two main problems: $(\RN{1})$ The function under consideration is often difficult to optimize and find a good local minima.  It is believed that this is in large part due to the pathological curvature and highly non-convex nature of the function to be optimized~\cite{dauphin2014identifying}. $(\RN{2})$  Standard optimization techniques lead to overfitting, typically addressed through early stopping~\cite{srivastava2014dropout}.

A Bayesian approach for learning neural networks incorporates uncertainty into model learning, and can reduce overfitting~\cite{mackay1992practical}.  In fact, it is possible to view the standard dropout technique~\cite{srivastava2014dropout} as a form of Bayesian approximation that incorporates uncertainty~\cite{Gal2015DropoutB,kingma2015variational}.  Many other recent works~\cite{blundell2015weight,hernandez2015probabilistic,korattikara2015bayesian} advocate incorporation of uncertainty estimates during model training to help improve robustness and performance.

While a Bayesian approach can ameliorate the overfitting issue in these complicated models, exact Bayesian inference in DNNs is generally intractable.  Recently, several approaches have been proposed to approximate a Bayesian posterior for DNNs, including a stochastic variational inference (SVI) method ``Bayes by Backprop'' (BBB)~\cite{blundell2015weight} and an online expectation propogation method (OEP) ``probabilistic backpropagation'' (PBP)~\cite{hernandez2015probabilistic}. These methods assume the posterior is comprised of separable Gaussian distributions. While this is a good choice for computational reasons, it can lead to unreasonable approximation errors and underestimation of model uncertainty.

A popular alternative to SVI and OEP is to use Stochastic Gradient Markov Chain Monte Carlo (SG-MCMC) methods to generate posterior samples~\cite{welling2011bayesian,ChenFG:ICML14,DingFBCSN:NIPS14,li2016thermostats}.  One of the most common SG-MCMC methods is the Stochastic Gradient Langevin Dynamics (SGLD) algorithm \cite{welling2011bayesian}.  One merit of this approach is that it is highly scalable; it requires only the gradient on a small mini-batch of data, as in the optimization method Stochastic Gradient Descent (SGD).  It has been shown that these MCMC approaches converge to the true posterior by using a slowly-decreasing sequence of step sizes~\cite{TehThiVol2014a,chen2015integrator}.  Costly Metropolis-Hasting steps are not required.

Unfortunately, DNNs often exhibit pathological curvature and saddle points~\cite{dauphin2014identifying}, which render existing SG-MCMC methods inefficient.  In the optimization literature, numerous approaches have been proposed to overcome this problem, including methods based on adapting a preconditioning matrix in SGD to the local geometry~\cite{duchi2011adaptive,kingma2014adam,dauphin2015rmsprop}.  These approaches estimate second-order information with trivial per-iteration overhead, have improved risk bounds in convex problems compared to SGD, and demonstrate improved empirical performance in DNNs.  The idea of using geometry in SG-MCMC has been explored in many contexts \cite{ahn2012bayesian,girolami2011riemann,patterson2013stochastic} and includes second-order approximations.  Often, these approaches use the expected Fisher information, adding significant computational overhead.  These methods lack the scalability necessary for learning DNNs, as discussed further below.  

We combine adaptive preconditioners from optimization with SGLD, to improve SGLD efficacy. To note the distinction from SGLD, we refer to this as the Preconditioned SGLD method (pSGLD).  This procedure is simple and adds trivial per-iteration overhead.  We first show theoretical properties of this method, including bounds on risk and asymptotic convergence properties.  We demonstrate improved efficiency of pSGLD by demonstrating an enhanced bias-variance tradeoff of the estimator for small problems.  We further empirically demonstrate its application to several models and large datasets, including deep neural networks.   In the DNN experiments, pSGLD outperforms the results based on standard SGLD from \cite{korattikara2015bayesian}, both in terms of convergence speed and the test-set performance.  Futher, pSGLD generates state-of-the-art performance for the examples tested.

\vspace{-2mm}
\section{Related Work}

Various regularization schemes have been developed to prevent overfitting in neural networks, such as early stopping, weight decay, dropout~\cite{srivastava2014dropout}, and dropconnect~\cite{wan2013regularization}.  Bayesian methods are appealing due to their ability to avoid overfitting by capturing uncertainty during learning~\cite{mackay1992practical}. MCMC methods work by producing Monte Carlo approximations to the posterior, with asymptotic consistency~\cite{neal1995bayesian}.  Traditional MCMC methods use the full dataset, which does not scale to large data problems.  A pioneering work in combining stochastic optimization with MCMC was presented in \cite{welling2011bayesian}, based on Langevin dynamics~\cite{neal2011mcmc}.  This method was referred to as Stochastic Gradient Langevin Dynamics (SGLD), and required only the gradient on mini-batches of data.  The per-iteration cost of SGLD is nearly identical to SGD.  Unlike SGD, SGLD can generate samples from the posterior by injecting noise into the dynamics.  This encourages the algorithm to explore the full posterior, instead of simply converging to a maximum {\em a posterior} (MAP) solution. Later, SGLD was extended by \cite{ahn2012bayesian}, \cite{patterson2013stochastic} and \cite{korattikara2015bayesian}. Furthermore, higher-order versions of the SGLD with momentum have also been proposed, including stochastic gradient Hamiltionian Monte Carlo (SGHMC)~\cite{ChenFG:ICML14} and stochastic gradient Nose-Hoover Thermostats (SGNHT)~\cite{DingFBCSN:NIPS14}.

It has been shown that incoporating higher-order gradient information helps train neural networks when employing optimization methods \cite{ngiam2011optimization}. However, calculations of higher-order information is often cumbersome in most models of interest. Methods such as quasi-Newton, and those approximating second-order gradient information, have shown promising results~\cite{ngiam2011optimization}. An alternative to full quasi-Newton methods is to rescale parameters so that the loss function has similar curvature along all directions. This strategy has shown improved performance in Adagrad~\cite{duchi2011adaptive}, Adadelta~\cite{zeiler2012adadelta}, Adam~\cite{kingma2014adam} and RMSprop~\cite{tieleman2012lecture} algorithms. Recently, RMSprop has been explained as a diagonal preconditioner in~\cite{dauphin2015rmsprop}. While relatively mature in optimization, these techniques have not been developed in sampling methods. In this paper, we show that rescaling the parameter updates according to geometry information can also improve SG-MCMC, in terms of both training speed and predictive accuracy.

\vspace{-2mm}
\section{Preliminaries}

Given data $\Dcal = \{\dv_i \}^N_{i=1}$, the posterior of model parameters $\thetav$ with prior $p(\thetav)$ and likelihood $\prod_{i=1}^N p(\dv_i  | \thetav)$
is computed as $p( \thetav | \Dcal ) \propto p(\thetav ) \prod_{i=1}^{N} p(\dv_i  | \thetav)$. In the optimization literature, the prior plays the role of a penalty that regularizes parameters, while the likelihood constitutes the loss function to be optimized. The task in optimization is to find the MAP estimate, $\thetav_{\textsl{MAP}} = \arg\!\max \log p(\thetav | \Dcal)$. Let $\Delta\thetav_t$ denote the change in the parameters at time $t$. Stochastic optimization methods such as Stochastic Gradient Descent (SGD)\footnote{For maximization, this is Stochastic Gradient {\em Ascent}. Here, we abuse notation because SGD is a more common term.} update $\thetav$ using the following rule:
\vspace{-1.5mm}
\beqs
\Delta	\thetav_{t} =  
\epsilon_t 
\Big (
\nabla_{\thetav} \log p(\thetav_t) 
+ \frac{N}{n}\sum_{i=1}^{n} \nabla_{\thetav} \log p(\dv_{t_i} | \thetav_t ) 
\Big )
\label{Eq:sgd_update} 
\eeqs
where $\{\epsilon_t\}$ is a sequence of step sizes, and 
$\Dcal^t = \{\dv_{t_1}, \cdots , \dv_{t_n}\}$ a subset of $n<N$ data items randomly chosen from $\mathcal{D}$ at iteration $t$.  The convergence of SGD has been established \cite{Bottou:04}.

For DNNs, the gradient is calculated by backpropagation~\cite{williams1986learning}.  One data item $ \dv_i \triangleq  (x_i, y_i) $ may consist of input $x_i \in \R^D $ and output $y_i \in \Ycal$, with $\Ycal$ being the output space (\eg a discrete label space in classification). In the testing stage, the Bayesian predictive estimate for input $x$, is given by $p(y|x,\Dcal)=\mathbb{E}_{p(\thetav|\Dcal)}[p(y|x,\thetav)]$.  The MAP estimate simply approximates this expectation as $p(y|x, \Dcal ) \approx p(y|x, \thetav_{\text{MAP}})$, ignoring parameter uncertainty.


Stochastic sampling methods such as SGLD incorporate uncertainty into predictive estimates.  SGLD samples $\thetav$ from the posterior distributions via a Markov Chain with steps:
\vspace{-2mm}
{\small
\begin{equation}
\Delta	\thetav_{t}   \sim
\Ncal
\left(
\frac{\epsilon_t }{2}
\Big (
\nabla_{\thetav}  \log p(\thetav_t)  
+ \frac{N}{n}\sum_{i=1}^{n} \nabla_{\thetav} \log p(\dv_{t_i} | \thetav_t ) 
\Big )
, \epsilon_t \Imat  
\right)
\label{Eq:sgld_update} 
\end{equation}}
with $\Imat$ denoting the identity matrix.  It also uses mini-batches to take gradient descend steps at each iteration. Rates of convergence are proven rigorously in \cite{TehThiVol2014a}. Given a set of samples from the update rule \eqref{Eq:sgld_update},
posterior distributions can be approximated via Monte Carlo approximations as $p(y|x, \Dcal ) \approx \frac{1}{T} \sum_{t=1}^{T} p(y|x, \thetav_t )$, where $T$ is the number of samples.

Both stochastic optimization and stochastic sampling approaches have the requirement that the step sizes satisfy the the following  assumption.\footnote{The requirement for SGLD can be relaxed, see \cite{TehThiVol2014a,chen2015integrator} for more details.}
\begin{assumption}\label{ass:stepsize_constraints}
	The step sizes $\{\epsilon_t\}$ are decreasing, {\it i.e.}, $0 < \epsilon_{t+1} < \epsilon_{t}$, with
	1) $\sum_{t=1}^{\infty}
	\epsilon_t 
	= \infty$; 
	and 2) $\sum_{t=1}^{\infty}
	\epsilon_t^2 
	< \infty$.
\end{assumption}
If these step-sizes are not satisfied in stochastic optimization, there is no guarantee of convergence because the gradient estimation noise is not eliminated.  Likewise, in stochastic sampling, decreasing step-sizes are necessary for asymptotic consistency with the true posterior, where the approximation error is dominated by the natural stochasticity of Langevin dynamics~\cite{welling2011bayesian}. 
%

\vspace{-1mm}
\section{Preconditioned Stochastic Gradient Langevin Dynamics} \label{sec:pre_sgld}
As noted in the previous section, standard SGLD updates all parameters with the same step size. This could lead to slow mixing when the components
of $\thetav$ have different curvature. Unfortunately, this is generally true in DNNs due to the composition of nonlinear functions at multiple layers. A potential solution is to employ a user-chosen preconditioning matrix $G(\thetav)$ in SGLD~\cite{girolami2011riemann}. The intuition is to consider the family of probability distributions $p( \dv | \thetav)$ parameterised by $\thetav$ lying on a Riemannian manifold. One can use the non-Euclidean geometry implied by this manifold to guide the random walk of a sampler. For any probability distribution, the expected Fisher information matrix $\Ical_{\thetav}$ defines a natural Riemannian metric tensor. To further scale up the method to a general online framework stochastic gradient Riemannian Langevin dynamics (SGRLD) was suggested in~\cite{patterson2013stochastic}. At position $\thetav_t$, 
it gives the step\footnote{The update form in \cite{patterson2013stochastic}
is more complicated and seemingly different from \eqref{Eq:sgrld_update}; however, they can be shown to be equivalent.}, 
\vspace{-1mm}
{\small
\begin{align}
\Delta\thetav_{t} & \sim	 
\frac{\epsilon_t}{2}   
\Big [  G(\thetav_t) 
\Big (
\nabla_{\thetav}  \log p(\thetav_t)   \label{Eq:sgrld_update}  \\
& + \frac{N}{n}\sum_{i=1}^{n} \nabla_{\thetav}  \log p(\dv_{t_i} | \thetav_t ) 
\Big )   
+ \Gamma(\thetav_t)
\Big ] 
+ G^{\frac{1}{2}}(\thetav_t) \Ncal(0, \epsilon_t \Imat  )
 \nonumber
\end{align}}
\!\!\!where $ \Gamma_i(\thetav)= \sum_j \frac{\partial G_{i,j} (\thetav)}{\partial \theta_{j}}$ describes how the preconditioner changes with respect to $\thetav_t$. . This term vanishes in SGLD because the preconditioner of SGLD is a constant $\Imat$. Both the direction and variance in Eq.(\ref{Eq:sgrld_update}) depends on the geometry of $G^{}(\thetav_t)$.  The natural gradient in the SGRLD step takes the direction of steepest descent on a manifold. Convergence to the posterior is guaranteed \cite{TehThiVol2014a,chen2015integrator} as long as step sizes satisfy Assumption~\ref{ass:stepsize_constraints}. 

Unfortunately, for many models of interest, the expected Fisher information is intractable. However, we note that any positive definite matrix defines a valid Riemannian manifold metric. Hence, we are not restricted to using the exact expected Fisher information. Preconditioning aims to constitute a local transform such that the rate of curvature is equal in all directions. Following this, we propose to use the same preconditoner as in RMSprop. This preconditioner is updated sequentially using only the current gradient information, and only estimates a diagonal matrix. It is given sequentially as,
\vspace{-2mm}
\begin{align}
&G( \thetav_{t+1} )=\text{diag} \left ( {\bf 1}  \oslash \big ( \lambda {\bf 1} + \sqrt{V( \thetav_{t+1} ) } \big ) \right )\\
&V( \thetav_{t+1} )=\alpha V( \thetav_{t} )+ (1-\alpha)  \bar{g}(\thetav_t; \Dcal^t  ) \odot  \bar{g}(\thetav_t; \Dcal^t )~,
\end{align}
%
where for notational simplicity, $\bar{g}(\thetav_t; \Dcal^t ) = \frac{1}{n}\sum_{i=1}^{n} \nabla_{\thetav}  \log p(\dv_{t_i} | \thetav_t ) $, is the sample mean of the gradient using mini-batch $\Dcal^t$, and $\alpha \in[0, 1]$. Operators $\odot$ and $ \oslash $ represent element-wise matrix product and division, respectively. 

RMSprop utilizes magnitudes of recent gradients to construct a preconditioner. 
Flatter landscapes have smaller gradients while curved landscapes have larger gradients. Gradient information is usually only locally consistent.
Therefore, two equivalent interpretations for Eq. \eqref{Eq:sgrld_update} can be reached intuitively: $\RN{1}$) the preconditioner equalizes the gradient so that a constant stepsize is adequate for all dimensions. $\RN{2}$) the stepsizes are adaptive, in that flat directions have larger stepsizes while curved directions have smaller stepsizes. 

In DNNs, saddle points are the most prevalent critical points, that can considerably slow down training~\cite{dauphin2015rmsprop}, mostly because the parameter space tends to be flat in many directions and ill-conditioned in the neighborhood of these saddle points. Standard SGLD will slowly escape the saddle point due to the typical oscillations along the high positive curvature direction. By transforming the landscape to be more equally curved, it is possible for the sampler to move much faster.

In addition, there are two tuning parameters: $\lambda$ controls the extremes of the curvature in the preconditioner (default $\lambda\!\!=\!\!10^{-5}$), and $\alpha$ balances the weights of historical and current gradients. We use a default value of $\alpha\!\!=\!\!0.99$ to construct an exponentially decaying sequence. Our Preconditioned SGLD with RMSprop is outlined in Algorithm~\ref{alg:sgld_rmsprop}.
\vspace{-8mm}
\begin{center}
	\begin{minipage}[t]{1.0\linewidth}
		\vspace{0pt}  
		\begin{algorithm}[H] 
			\caption{Preconditioned SGLD with RMSprop}\label{alg:sgld_rmsprop}
			\small
			\begin{algorithmic}
				\STATE {\bf Inputs:} $\{\epsilon_t\}_{t=1:T}, \lambda, \alpha$
				\STATE {\bf Outputs:} $\{\thetav_t\}_{t=1:T}$
				\STATE {\bf Initialize:} $\Vmat_0 \leftarrow {\bf 0}$, random $\thetav_1$
				\FOR{$t\leftarrow 1 : T$}
				\STATE Sample a minibatch of size $n$, $\Dcal_n^t=\{\dv_{t_1}, \dots, 
				\dv_{t_n} \}$
				\STATE Estimate gradient 
				$\bar{g}(\thetav_t; \Xmat^t ) = \frac{1}{n}\sum_{i=1}^{n} \nabla \log p(\dv_{t_i} | \thetav_t ) $

				\STATE $
				V( \thetav_{t} ) \leftarrow 
				\alpha V( \thetav_{t-1} )+ (1-\alpha)  \bar{g}(\thetav_t; \Dcal^t  ) \odot  \bar{g}(\thetav_t; \Dcal^t )   		
				$ 
				\STATE 
				$G( \thetav_{t} )\leftarrow  \text{diag} \left ({\bf 1}  \oslash \big (  \lambda {\bf 1} + \sqrt{V( \thetav_{t} ) } \big )  \right )$
				\STATE
				$
				\thetav_{t+1} \leftarrow  	\thetav_{t}  + 
				\frac{\epsilon_t}{2}   
				\Big [  G^{}(\thetav_t) 
				\Big (
				\nabla_{\thetav}  \log p(\thetav_t) 
				+N  \bar{g}(\thetav_t; \Dcal^t  )
				\Big )
				+ \Gamma(\thetav_t)
				\Big ]
				+  \Ncal(0, \epsilon_t G^{}(\thetav_t) )
				$
				\ENDFOR
			\end{algorithmic}
		\end{algorithm}
	\end{minipage}%
\end{center}

\vspace{-3mm}
\section{Preconditioned SGLD Algorithms in Practice}
This section first analyzes  the finite-time convergence properties of pSGLD, then proposes a more efficient variant for practical use.  We note that prior work gave similar theoretical results~\cite{chen2015integrator}, and we extend the theory to consider the use of preconditioners.
\vspace{-1mm}
\subsection{Finite-time Error Analysis}
For a bounded function $\phi(\thetav)$, 
we are often interested in its true posterior expectation $\bar{\phi} = \int_{\mathcal{X}} \phi(\thetav)  p (\thetav | \Dcal ) d \thetav$. For example, the class distribution of a data point in DNNs.
In our SG-MCMC based algorithm, this intractable integration is approximated by a weighted sample average 
$\hat{\phi} = \frac{1}{S_T} \sum_{t=1}^{T} \epsilon_t\phi(\thetav_t)$ at time $S_T = \sum_{t=1}^T \epsilon_t$,
with stepsizes $\{\epsilon_t\}$. These samples are generated from an MCMC algorithm with a numerical integrator ({\it e.g.},
our pSGLD algorithm) that discretizes the continuous-time Langevin dynamics.  The precision
of the true posterior average and its MCMC approximation is characterized by the expected difference between $\bar{\phi}$ and $\hat{\phi}$. We analyze the pSGLD algorithm by extending the work of \cite{TehThiVol2014a,chen2015integrator} to include
adaptive preconditioners.  We first show the asymptotic convergence properties of our algorithm in Theorem~\ref{theo:mse_decrease}
by the mean of the mean squared error (MSE)\footnote{This is different from the optimization literature where the {\em regret} is studied, which is not straightforward in the MCMC framework.}. To get the convergence result, some mild assumptions on the smoothness
and boundness of $\psi$, the solution functional of $\mathcal{L}\psi = \phi(\thetav_t) - \bar{\phi}$, is needed,
where $\mathcal{L}$ is the generator of corresponding stochastic differential equation for pSGLD. We discuss these conditions and prove the Theorem in Appendix A.

\begin{theorem}\label{theo:mse_decrease}
	Define the operator $\Delta V_t = (N \bar{g}(\thetav_t; \Dcal^t) - g(\thetav_t; \Dcal^t))^\top G^{} (\thetav_t) \nabla_{\thetav}$.
	Under Assumption~\ref{ass:stepsize_constraints}, for a test function $\phi$, the MSE 
	of the pSGLD at finite time 
	$S_T$ is bounded, for some $C > 0$ independent of $\{ \epsilon_t \}$, as:
	\begin{align}
	\text{MSE} : &
	\mathbb{E} \left[ \left(\hat{\phi} - \bar{\phi}\right)^2 \right]
	 \leq  \mathcal{B}_{\text{mse}}\label{eq:mse_decrease} \\
	& \triangleq C\left(\sum_t \frac{\epsilon_t^2}{S_T^2}\mathbb{E}\left\|\Delta V_t\right\|^2 + \frac{1}{S_T} + \frac{(\sum_{t=1}^T \epsilon_t^{2})^2}{S_T^2} \right)~. \nonumber
	\end{align}
\end{theorem}

MSE is a common measure of quality of an estimator, reflecting the precision of an approximate algorithm. 
It can be seen from Theorem~\ref{theo:mse_decrease} that the finite-time approximation error of pSGLD is bounded by $\mathcal{B}_{\text{mse}}$, consisting of two factors: $(\RN{1})$ estimation error from stochastic gradients, $\sum_t \frac{\epsilon_t^2}{S_T^2} \mathbb{E}\left\|\Delta V_t\right\|^2$, and $(\RN{2})$ discretization error inherited 
from numerical integrators,  $\frac{1}{S_T} + \frac{(\sum_{t=1}^T \epsilon_t^{2})^2}{S_T^2}$. These terms asymptotically approach $0$ under 
Assumption~\ref{ass:stepsize_constraints}, 
meaning that the decreasing-step-size pSGLD is asymptotically consistent with true posterior expectation.
\vspace{-5pt}
\subsection{Practical Techniques}
Of interest when considering the practical issue of limited computation time, we now interpret the above finite-time error using the 
framework of \textit{risk of an estimator}, which provides practical guidance in implementation. 
From \cite{korattikara2014austerity}, the predictive risk, $R$, of an algorithm is defined as the MSE
above, and can be decomposed as $R = \E [ ( \bar{\phi}  - \hat{\phi})^2] = B^2 + V$ , where $B$ is the bias 
and $V$ is the variance. Denote $\bar{\phi}_{\eta} = \int_{\mathcal{X}} \phi(\thetav)  \rho_{\eta}(\thetav) d \thetav$
as the ergodic average under the invariant measure, $ \rho_{\eta}(\thetav)$, of the pSGLD. After burnin, it can be 
shown that
\beqs
\text{Bias} &:& B = \bar{\phi}_{\eta} - \bar{\phi} \label{eq:bias}
\\
\text{Variance} &:& V = \E [ (  \bar{\phi}_{\eta} -  \hat{\phi})^2 ]
\approx  \frac{   A_{} (0) }{ M_{\eta} } \label{eq:varance}
\eeqs
where $ A_{} (0) $ is the variance of $\phi$ with respect to $\rho_{\eta}(\thetav)$ (\ie $\mathbb{E}_{\rho_{\eta}(\thetav)}[(\phi-\hat{\phi})^2]$) , which is a constant (further details are given in Appendix D).
$M_{\eta}$ is the \textit{effective sample size} (ESS), defined as
\beqs
\text{ESS} &:& M_{\eta} = 
\frac{ T }{ 1 + 2 \sum_{t=1}^{\infty} 	\frac{A( t ) }{A( 0 )  }  }
= \frac{ T }{  2 \tau }
\label{eq:ess}
\eeqs
where $A(t) = 
\E [ (  \bar{\phi}_{\eta} -  \phi(\thetav_0) )
(  \bar{\phi}_{\eta} -  \phi(\thetav_t) )
]$ is the \textit{autocovariance function}, manifesting how strong two samples with a time lag $t$ are correlated. The term $\tau=  \frac{1}{2} +  \sum_{t=1}^{\infty} 	\frac{A( t ) }{ A( 0 )  }   $ is the integrated \textit{autocorrelation time} (ACT), which measures the interval between independent samples.  

In practice, there is always a tradeoff between bias and variance.  In the case of infinite computation time, the traditional MCMC setting can reduce the bias and variance to zero.  However, in practice, time is limited. Obtaining more effective samples can reduce the total risk (Eq.~\eqref{eq:mse_decrease}), even if bias is introduced. In the following, we provide two model-independent practical techniques to further speed up the proposed pSGLD.
\vspace{-3mm}
\paragraph{Excluding $\Gamma(\thetav_t)$ term}
Though the evaluation of $\Gamma(\thetav_t)$ in our case is manageable due to its diagonal nature, we propose to remove it during sampling to reduce the computation.
It is interesting that in our case ignoring $\Gamma(\thetav_t)$ produces a bias controlled by $\alpha$ on the MSE.

\begin{corollary}\label{coro:gamma_term}
	Assume the 1st-order and 2nd-order gradients are bounded.  With the same assumptions as Theorem~\ref{theo:mse_decrease},
	the MSE when ignoring the $\Gamma(\thetav_t)$ term in the algorithm can be bounded as $\mathbb{E} \left[ \left(\hat{\phi} - \bar{\phi}\right)^2 \right] \leq \mathcal{B}_{\text{mse}} + O\left(\frac{(1 - \alpha)^2}{\alpha^3}\right)$, where $\mathcal{B}_{\text{mse}}$ is the bound defined in Theorem~\ref{theo:mse_decrease}.
\end{corollary}

Omitting $\Gamma(\thetav_t)$ introduces an extra term in the bound that is controlled by the parameter $\alpha$.  
The proof is in Appendix B. Since  $\alpha$ is always set to a value that is very close to $1$, the term $ (1- \alpha)^2/ \alpha^3  \approx 0$, the effect of $\Gamma(\thetav_t)$ negligible. In addition, more samples per unit time are generated when  $\Gamma(\thetav_t)$ is ignored, 
resulting in a smaller variance on the prediction. Note that the term $\Gamma(\thetav_t)$ is heuristically ignored in~\cite{ahn2012bayesian}, but is only able to approximate the true posterior in the case of infinite data, which is not required in our algorithm. 

\vspace{-3mm}
\paragraph{Thinning samples}
Making predictions using a whole ensemble of models is cumbersome and may be too 
computationally expensive to allow deployment for a large number of users, especially when 
the models are large neural nets. One practical technique is to average models using a thinned
version of samples. By thinning the samples in pSGLD, the total number of samples is reduced.  However, these thinned samples have a lower autocorrelation time and can have a similar ESS.  We can also guarantee the MSE remains the same form under the thinning schema. The proof is in Appendix C.


\begin{corollary}\label{coro:thin}
	By thinning samples from our pSGLD algorithm, the MSE remains the same form as in Theorem~\ref{theo:mse_decrease},
	and asymptotically approaches 0. 
\end{corollary}

%
%
%

\vspace{-10pt}
\section{Experiments}

Our experiments focus on the effectiveness of preconditioners in pSGLD, and present results in four parts: a simple simulation, Bayesian Logistic Regression (BLR), and two widely used DNN models, Feedforward Neural Networks (FNN) and Convolutional Neural Networks (CNN). 

The proposed algorithm that uses the discussed practical techniques is denoted as \textit{pSGLD}. The prior on the parameters is set to $p(\thetav) = \Ncal(0, \sigma^2 \Imat )$. If not specifically mentioned, the default setting for DNN experiments is shared as follows. $\sigma^2 = 1 $, minibatch size is 100, thinning interval is 100, burn-in is 300. We employ a block decay strategy for stepsize; it decreases by half after every $L$ epochs.

\vspace{-3pt}
\subsection{Simulation}
We first demonstrate pSGLD on a simple 2D Gaussian example,
{\tiny	$\Ncal \Big( \left [ \begin{array}{c} 0 \\ 0\end{array}  \right ], 
	\left [ \begin{array}{cc} 0.16&0 \\ 0&1 \end{array} \right ] 
	\Big)$}. Given posterior samples, the goal is to estimate the covariance matrix. A diagonal covariance matrix is used to show the algorithm can adjust the stepsize at different dimension.

We first compare SGLD and pSGLD with a large range of different stepsizes $\epsilon$. $2\times10^5$ samples are collected for each algorithm. Reconstruction errors and autocorrelation time are shown in Fig.~\ref{fig:Gaussian} (a). 
We see that pSGLD dominates the ``vanilla'' SGLD in that it consistently shows a lower error and autocorrelation time, particularly with larger stepsize. When the stepsize is small enough, the sampler does not move much, and the performances of the two algorithms become similar. The first $600$ samples of both methods for  $\epsilon = 0.3$ are shown in Fig.~\ref{fig:Gaussian} (b). Because step sizes in pSGLD
can be adaptive, it implies that even if the covariance matrix of a target distribution is mildly rescaled, a new stepsize is unnecessary for pSGLD. Meanwhile, the stepsize of the vanilla SGLD needs to be fine-tuned in order to obtain decent samples.  See Appendix E for further details.
\vspace{-10pt}
\begin{figure}[h] \centering
	\begin{tabular}{c c} \hspace{-0.3cm}
		\includegraphics[width=4.2cm]{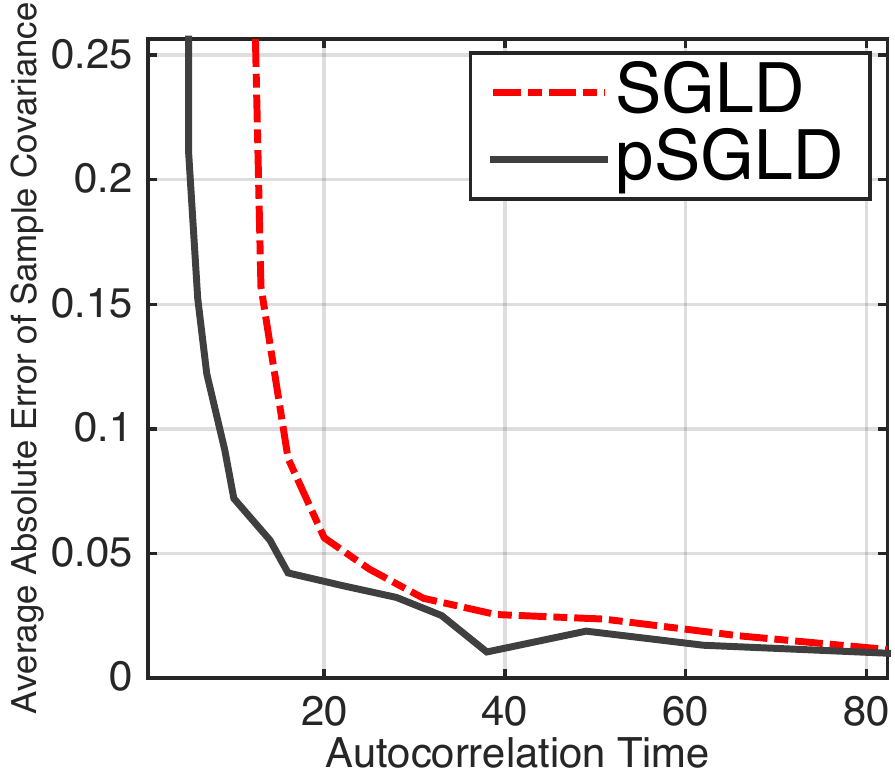} \vspace{5pt}&  \hspace{-0.5cm}
		\includegraphics[width=4.1cm]{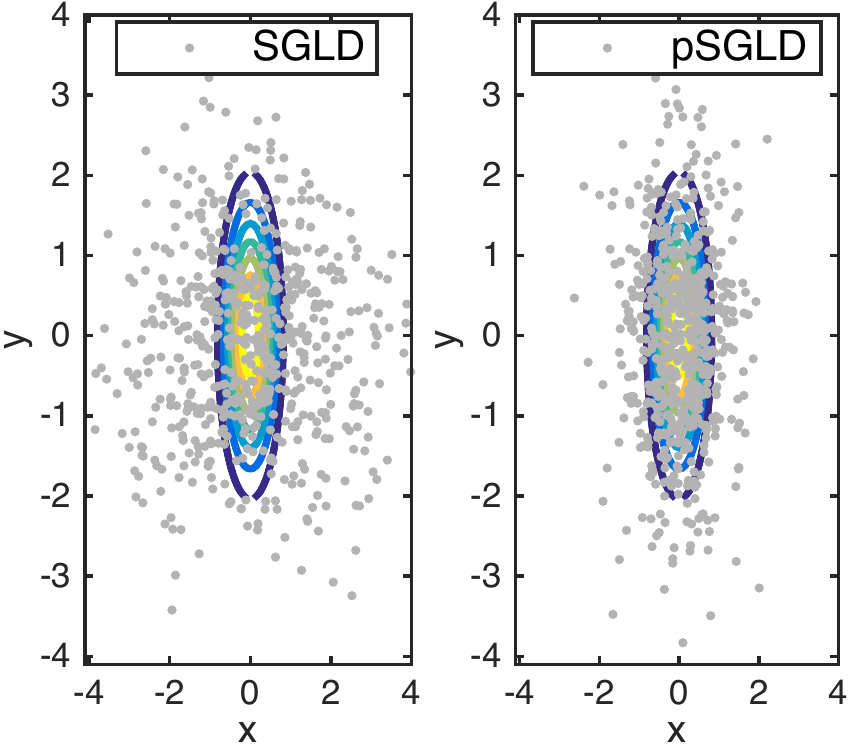}
		\vspace{-3pt}
		 \\ \vspace{-0mm}
		(a) Error and ACT
		&
		(b) Samples
	\end{tabular}
	\vspace{-4mm}
	\caption{Simulation results on a 2D Gaussian.}
	\label{fig:Gaussian}
	\vspace{-18pt}
\end{figure}
\vspace{-3pt}
\subsection{Bayesian Logstic Regression} 

To demonstrate that our pSGLD is applicable to general Bayesian posterior sampling, we demonstrate results on BLR. A small Australian dataset~\cite{girolami2011riemann} is first used with $N= 690$ and dimension $D= 14$. We choose a minibatch size of $5$.  The prior variance is $\sigma^2 = 100$.  $5\times10^3$ iterations are used. For both pSGLD and SGLD, we test stepsize  $\epsilon$ ranging from $1 \times 10^{-7} $ to $ 1 \times 10^{-4} $, with 50 runs for each algorithm.
\vspace{-12pt}
\begin{figure}[h] \centering
	\begin{tabular}{c c} \hspace{-0.1cm}
		\includegraphics[width=3.7cm]{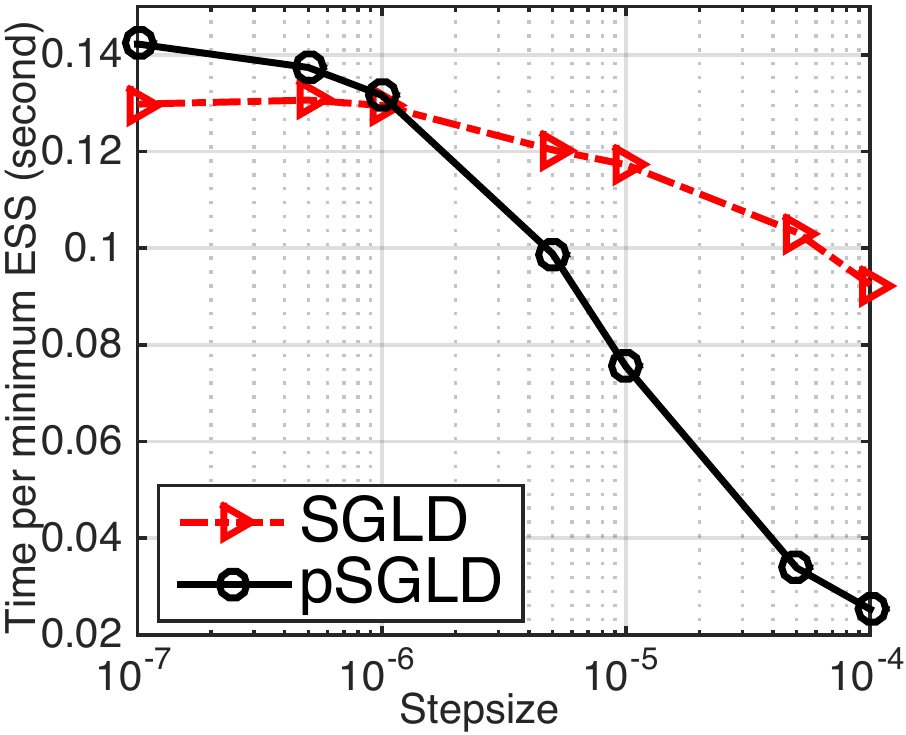} \vspace{5pt}&  \hspace{-3mm}
		\includegraphics[width=3.6cm]{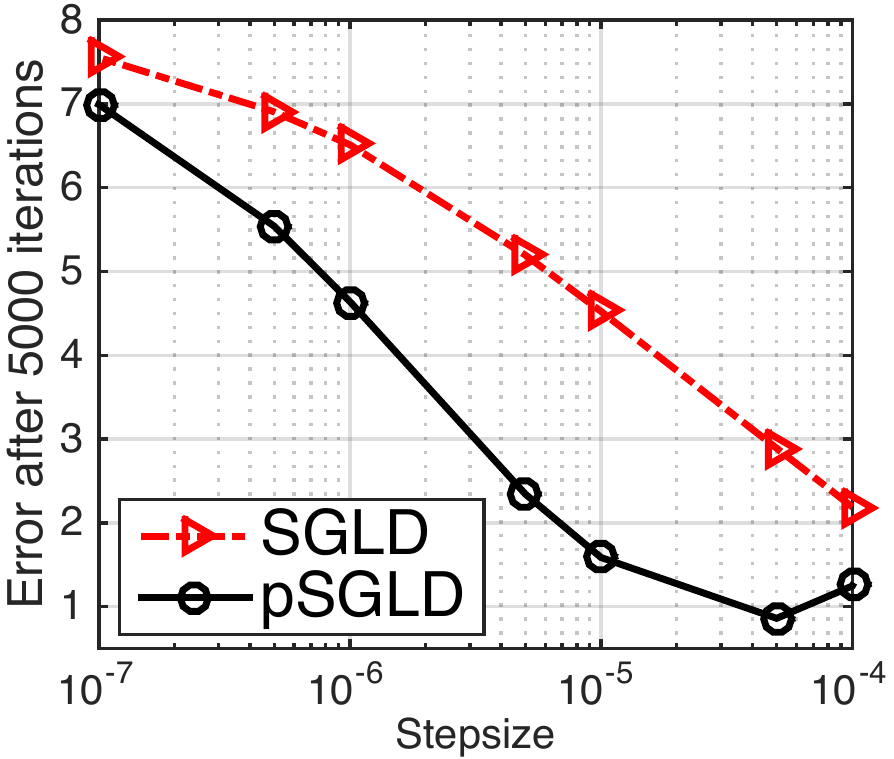}
		\vspace{-8pt}
		\\ \vspace{-0mm}
		(a) Variance
		&
		(b) Parameter estimation
	\end{tabular}
	\vspace{-10pt}
	\caption{BLR on Australian dataset.}
	\label{fig:blr_ess}
	\vspace{-12pt}
\end{figure}
%

Following~\cite{girolami2011riemann}, we report the time per minimum Effective Sample ($\propto1/\text{EES}$) in Fig.~\ref{fig:blr_ess} (a), which is proportional to the variance. pSGLD generates much larger ESS compared to SGLD, especially when the stepsize is large. Meanwhile, Fig.~\ref{fig:blr_ess} (b) shows that pSGLD provides smaller error in estimating weights, where the ``groundtruth'' is obtained by $10^6$ samples from HMC with Metroplis-Hastings.
Therefore, the overall risk is reduced. 




We then test BLR on a large-scale Adult dataset, $\mathtt{a9a}$~\cite{lin2008trust}, with $N_{\text{train}}= 32561, N_{\text{test}}=16281,$ and $D=123$.
Minibatch size is set to 50, and the prior variance is $\sigma^2 = 10$.
The thinning interval is 50, burn-in is 500, and $T = 1.5\times10^4$ iterations are used.			
Stepsize $\epsilon = 5 \times 10^{-2} $ for pSGLD and  SGLD. The test errors are compared in Table~1\ref{tab:a9a_blr}, and learning curves are shown in Fig.~3\ref{fig:a9a_blr}. Both SG-MCMC methods outperform the recently proposed doubly stochastic variational Bayes (SDVI)~\cite{titsias2014doubly}, and higher-order variational autoencoder methods (L-BFGS-SGVI, HFSGVI)~\cite{fan2015secondVAE}. Furthermore, pSGLD converges in less than $4\times10^3$ iterations, while SGLD at least needs double the time to reach this accuracy.


\vspace{-2pt}
\hspace{-10pt}
\begin{minipage}{1\linewidth}\centering
	\hspace{-5pt}
	\begin{minipage}[b]{0.45\linewidth}
		\centering
		\small
		\begin{tabular}{cc}
			\hline
			Method & Test error \\
			\hline 
			pSGLD     &  {\bf 14.85\%}\\
			SGLD	   &  {\bf 14.85\%}\\		  
			\hline			
			DSVI   &  15.20\%\\		 	
			L-BFGS-SGVI   &  14.91\%\\		 	
			HFSGVI   &  15.16\%\\		 	
			\hline 
		\end{tabular}
		\vspace{5pt}
		\label{tab:a9a_blr}
		\captionof{table}{BLR on $\mathtt{a9a}$.}
	\end{minipage}
	\hspace{5mm}
	\begin{minipage}[b]{0.48\linewidth}
		\centering
		\includegraphics[width=4.05cm]{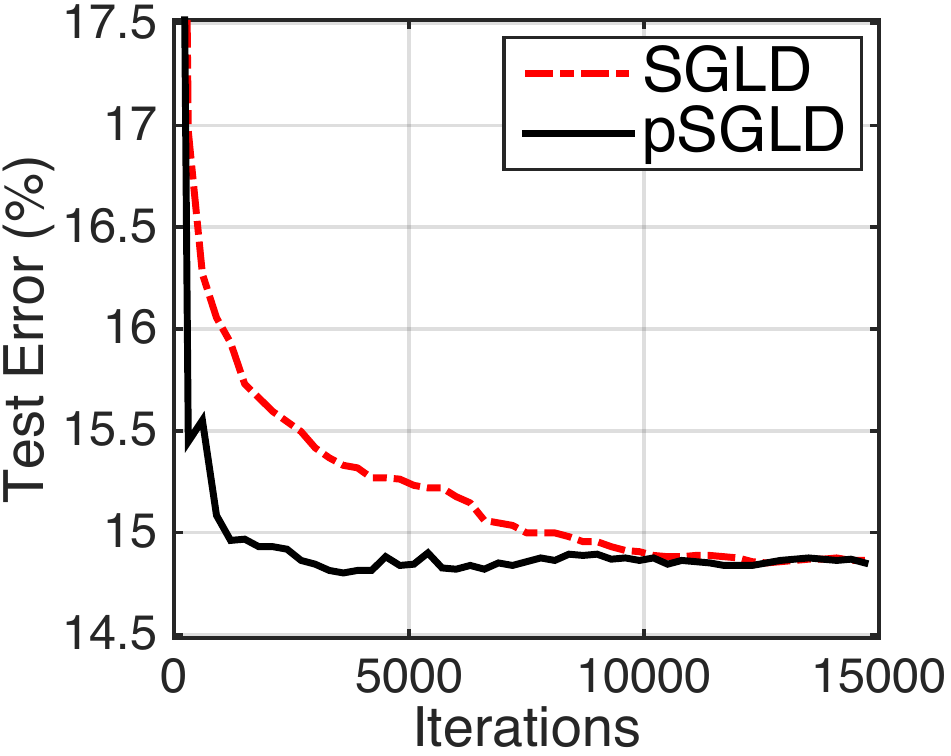}
		\label{fig:a9a_blr}
		\vspace{-20pt}
	    \captionof{figure}{Learning curves.}
	\end{minipage}
\end{minipage}


\vspace{-5pt}
\subsection{Feedforward Neural Networks}

The first DNN model we study is the Feedforward Neural Networks (FNN), or multilayer perceptron (MLP). The activation function is rectified linear unit (ReLU). A two-layer model, 784-X-X-10, is employed, where X is the number of hidden units for each layer.  100 epochs are used, with $L=20$.
We compare our propose method, pSGLD, with representative stochastic optimization methods: SGD, RMSprop and RMSspectral~\cite{carlson2015pssd}. After tuning, we set the optimal stepsize for each algorithm as: for pSGLD and RSMprop as follows: $\epsilon = 5 \times 10^{-4} $, while for SGLD and SGD as $\epsilon = 5 \times 10^{-1} $. 

We test the algorithms on the standard MNIST dataset, consisting of $28 \times 28$ images (thus the 784-dimensional input vector) from $10$ different classes ($0$ to $9$) with $60,000$ training and $10,000$ test samples. The test classification errors for network (X-X) size  400-400, 800-800 and 1200-1200 are shown in Table~\ref{tab:fnn}. The results  of stochastic sampling methods are better than their corresponding stochastic optimization counterparts. This indicates that incorporating weight uncertainty can improve performances. By increasing the variance $ \sigma^2 $ of pSGLD from  $1$  to $100$, more uncertainty is introduced into the model from the prior, and higher performance is obtained. Figure~\ref{fig:fnn1200} (a) displays the distribution histograms of weights in the last training iteration of the 1200-1200 model. We observe that smaller variance in the prior imposes lower uncertainty, by making the weights concentrate to $0$; while larger variance in the prior leads to a wider range of weight choices, thus higher uncertainty.

\begin{table} 
	\centering
	\caption{Classification error of FNN on MNIST. [ ${\diamond}$ ] indicates results taken from~\cite{blundell2015weight}	}\vspace{-0.3cm}
	\begin{minipage}{0.95\linewidth}
		\vskip 0.1in
		\centering
		\begin{adjustbox}{scale=0.9,tabular=llccc,center}
			\hline
			\multirow{2}{*}{Method } & \multicolumn{3}{c}{Test Error}  \\
			& 400-400 & 800-800  & 1200-1200 \\
			\hline 			 
			pSGLD  ($ \sigma^2 = 100 $)  & {\bf 1.40\%} & {\bf  1.26\%} & {\bf  1.14\%}\\		
			pSGLD 	($ \sigma^2 = 1 $)   & 1.45\% & {1.32\%}    &  {1.24\%} \\
			distilled pSGLD   & 1.44\% & {1.40\%}    &  {1.41\%} \\
			\hline 				
			SGLD	    	 &  1.64\% &    1.41\% &  1.40\%\\				  			
			RMSprop  	   &  1.59\% &    1.43\%  &  1.39\%\\	
			RMSspectral 	   		  &  1.65\% &    1.56\%  &   1.46\%\\			
			SGD 	   		  &   1.72\% &   1.47\%   &  1.47\%\\
			\hline  
			BPB, Gaussian$^\diamond$    &  1.82\% & 1.99\%      &  2.04\%\\
			BPB, Scale mixture$^\diamond$  &  1.32\% & 1.34\%  & 1.32\%\\			
			SGD, dropout$^\diamond$   & 1.51\%  & 1.33\%         & 1.36\%\\
			\hline
		\end{adjustbox}
		\label{tab:fnn}
	\end{minipage}
	\vspace{-10pt}
\end{table}

\begin{figure}[t!] \centering
	\hspace{-3mm}
	\begin{tabular}{c c}
		\includegraphics[height=3.2cm, width=4.1cm]{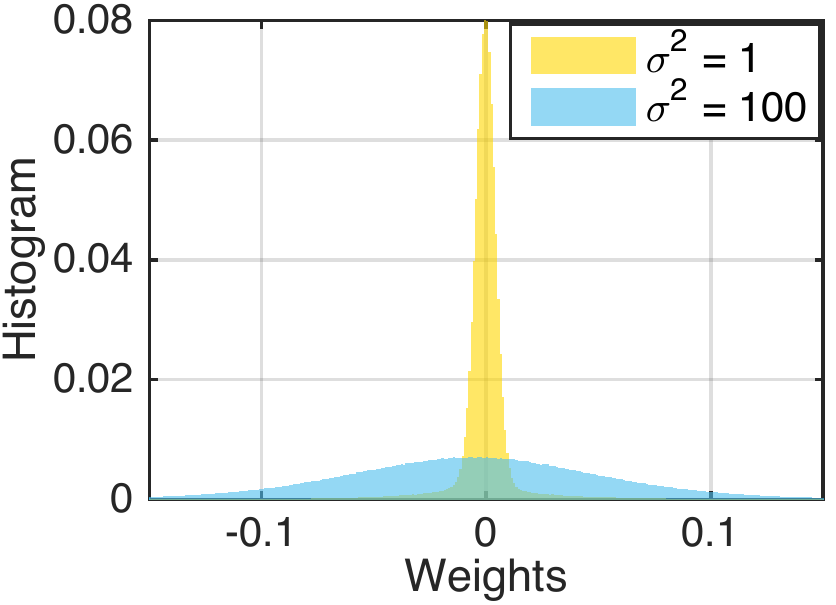} & \hspace{-3mm}
		\includegraphics[height=3.2cm, width=3.9cm]{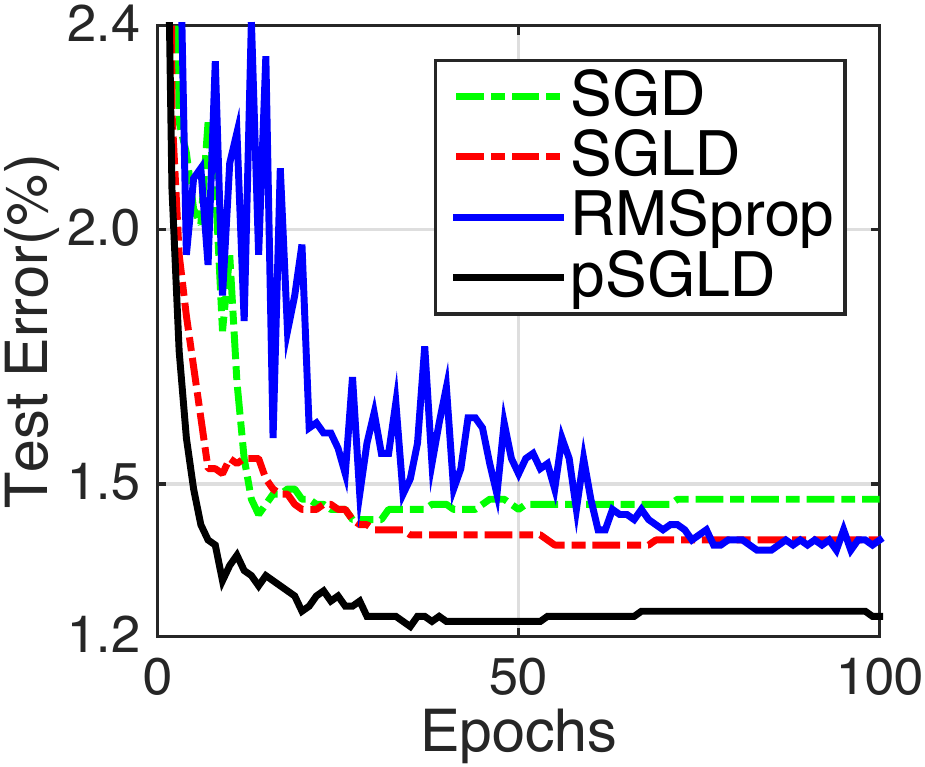} \\
		(a) Weights distribution & (b) Learning curves
	\end{tabular} \vspace{-3mm}
	\caption{FNN of size 1200-1200 on MNIST.}
	\label{fig:fnn1200}
	\vspace{-10pt}
\end{figure}

We also compare to other techniques developed to prevent overfitting (dropout) and weight uncertainty (BPB, Gaussian and scale mixtures).  pSGLD provides state-of-the-art performance for FNN on test accuracy. We further note that pSGLD is able to give increasing performance with increasing network size, whereas BPB and SGD dropout do not. This is probably because overfitting is harder to be dealt with in large neural networks with pure optimization techniques.

Finally, learning curves of network configuration 1200-1200 are plot in Fig.~\ref{fig:fnn1200} (b)\footnote{RMSspectral is not shown because it uses larger batch sizes and so is difficult to compare on this scale.}.
We empirically find that pSGLD and SGLD take fewer iterations to converge, and the results are more stable than their optimization counterparts.
Moreover, it can be seen that pSGLD consistently converges faster and to a better point than SGLD.  Learning curves for other network sizes are provided in Appendix F.  While the ensemble of samples requires more computation than a single FNN in testing, it shows significantly improved performance. As well, \cite{korattikara2015bayesian} showed that learning a single FNN that approximates the model average result gave nearly the same performance. We employ this idea, and suggest a fast version, {\em distilled pSGLD}. Its results for $ \sigma^2 = 1$ show it can maintain good performances.

\vspace{-5pt}
\subsection{Convolutional Neural Networks}
Our next DNN is the popular CNN model. We use a standard network configuration with 2 convolutional layers followed by 2 fully-connected layers~\cite{jarrett2009best}. Both convolutional layers use $5 \times 5$ filter size with 32 and 64 channels, respectively; $2 \times 2$ max pooling is used after each convolutional layer. The fully-connected layers have 200-200 hidden nodes with ReLU nonlinearities,
20 epochs are used, and $L$ is set to 10. The stepsizes for pSGLD and RMSprop is set to $\epsilon = \{1, 2\} \times 10^{-3} $ via grid search. 
For SGLD and SGD, this is $\epsilon =  \{1, 2\} \times 10^{-1} $. 
Additional results with CNNs are in Appendix G.

The same MNIST dataset is used. A comparison of test errors is shown in Table~3\ref{tab:cnn}, with the corresponding learning curves in Fig.~5\ref{fig:cnn}. We emphasize that the purpose of this experiment is to compare methods on the same model architecture, not to achieve overall state-of-the-art results.  The CNN trained with traditional SGD gives an error of 0.82\%.  pSGLD shows significant improvement, with an error of 0.45\%. This result is also comparable with some recent state-of-the-art CNN based systems, which have much more complex architectures. These include the stochastic pooling~\cite{zeiler2013stochastic}, Network in Network (NIN)~\cite{LinCY14NIN} and Maxout Network(MN)~\cite{goodfellow2013maxout}.

\vspace{-0mm}
\begin{minipage}{0.48\textwidth}\centering
	\hspace{-5mm} \vspace{0mm}
 	\begin{minipage}[b]{0.40\textwidth}
 		\centering
 		\small
		\begin{tabular}{cc}
			\hline
			Method & Test error \\
			\hline
			pSGLD                   & {\bf 0.45\%} \\
			SGLD				     &  0.71\%\\		  			
			RMSprop	    		   &  0.65\%\\
			RMSspectral & 0.78\%\\
			SGD				          &  0.82\%\\	\hline
			Stochastic Pooling		&  0.47\%\\	
			NIN + Dropout	&  0.47\%\\			 	
			MN + Dropout		&  0.45\%\\	
			\hline
		\end{tabular}
		\vspace{0mm}
		\label{tab:cnn}
 		\captionof{table}{Test error.}
 	\end{minipage}
 	 	\hfill \hfill
 	\begin{minipage}[b]{0.55\textwidth}
 		\centering
 		\includegraphics[width=4.05cm]{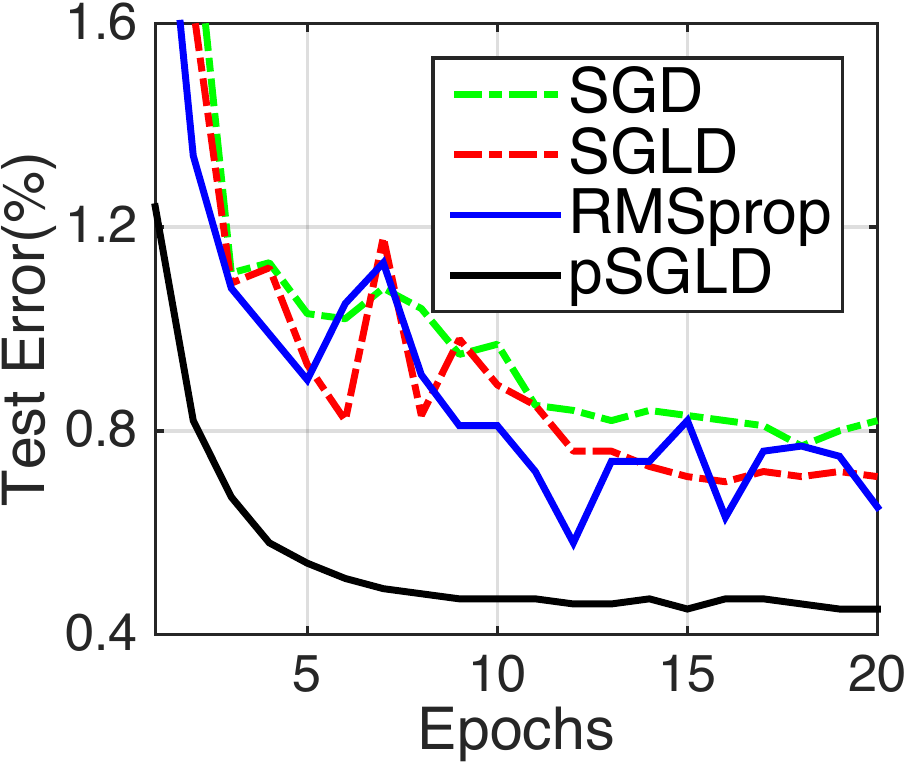}
 		\label{fig:cnn}
 		\vspace{-2mm}
 		\captionof{figure}{Learning curves.}
 	\end{minipage}
\end{minipage}

\vspace{-0pt}
\section{Conclusion}
A preconditioned SGLD is developed based on the RMSprop algorithm, with controllable finite-time approximation error. We apply the algorithm to DNNs to overcome their notorious problems of overfitting and pathological curvature. Extensive experiments show that our pSGLD can adaptive to the local geometry, allowing improved effective sampling rates and performance.  It provides sample-based uncertainty in DNNs, and achieves state-of-the-arts performances on FNN and CNN models.
%
Interesting future directions include exploring applications to latent variable models or recurrent neural networks~\cite{gan2015deep}.

\newpage
\paragraph{Acknowledgements} This research was supported in part by ARO, DARPA, DOE, NGA, ONR and NSF.

\bibliographystyle{aaai}	
\bibliography{subtex/references.bib}

\newpage

\appendix
\input{psgld_supp}

\end{document}

%% file: subtex/mlVecMat.tex
\usepackage{hyperref}
\usepackage{url}

\usepackage{amsmath}
\usepackage{amsfonts}
\usepackage{amssymb}
\usepackage{wrapfig}
\usepackage{bm}
\usepackage{subcaption}
\usepackage{psfrag}
\usepackage{epstopdf}
\usepackage{multirow}
\usepackage{fullpage}
\usepackage{mathtools}

\usepackage{color}
\usepackage{tikz}
\usetikzlibrary{arrows,shapes,snakes,automata,backgrounds,fit,petri}
\usepackage{adjustbox}

\usepackage{algorithm}
\usepackage{algorithmic}

\newcommand{\ie}[0]{\emph{i.e., }}

\newcommand{\eg}[0]{\emph{e.g., }}

\newcommand{\beq}{\vspace{0mm}\begin{equation}}
\newcommand{\eeq}{\vspace{0mm}\end{equation}}
\newcommand{\beqs}{\vspace{0mm}\begin{eqnarray}}
\newcommand{\eeqs}{\vspace{0mm}\end{eqnarray}}
\newcommand{\barr}{\begin{array}}
\newcommand{\earr}{\end{array}}

\newcommand{\Amat}[0]{{{\bf A}}}
\newcommand{\Bmat}{{\bf B}}

\newcommand{\Imat}{{\bf I}}

\newcommand{\Vmat}[0]{{{\bf V}}}

\newcommand{\Xmat}[0]{{{\bf X}}}

\newcommand{\av}[0]{{\boldsymbol{a}}}
\newcommand{\bv}[0]{{\boldsymbol{b}}}

\newcommand{\dv}{\boldsymbol{d}}

\newcommand{\thetav}{\boldsymbol{\theta}}

\newcommand{\R}{\mathbb{R}}
\newcommand{\E}{\mathbb{E}}

\newcommand{\Ycal}{\mathcal{Y}}
\newcommand{\Lcal}{\mathcal{L}}
\newcommand{\Ncal}{\mathcal{N}}

\newcommand{\Dcal}{\mathcal{D}}

\newcommand{\Ical}{\mathcal{I}}

\newtheorem{theorem}{Theorem} 
\newtheorem{lemma}[theorem]{Lemma}

\newtheorem{corollary}[theorem]{Corollary}

\ifx\assumption\undefined
\newtheorem{assumption}{Assumption}
\fi

\newenvironment{proof}[1][Proof]{\begin{trivlist}
\item[\hskip \labelsep {\bfseries #1}]}{\end{trivlist}}

\newcommand{\qed}{\nobreak \ifvmode \relax \else
      \ifdim\lastskip<1.5em \hskip-\lastskip
      \hskip1.5em plus0em minus0.5em \fi \nobreak
      \vrule height0.75em width0.5em depth0.25em\fi}

%% file: psgld_supp.tex
%
%
%
%
%
%
%

\twocolumn[
\begin{center}
	\bf{\LARGE Supplementary Material of \\Preconditioned Stochastic Gradient Langevin Dynamics for \\Deep Neural Networks\\ \vspace{1em}}
\end{center}
]

\section{A.$\quad$The proof for main theorem}
In \cite{chen2015integrator}, the authors provide the convergence property for general SG-MCMC, here we follow their assumptions and proof techniques, with specific treatment on the 1st-order numerical integrator, and the case of preconditioner.

\subsection{Details on the assumption}
Before the proof, we detail the assumptions needed for Theorem~1. 
For pSGLD, its associated Stochastic Differential Equation (SDE)  has an invariant measure $\rho(\thetav)$, the posterior average is defined as:
$\bar{\phi} \triangleq \int_{\mathcal{X}} \phi(\thetav) \rho(\thetav) \mathrm{d}\thetav$ 
for some test function $\phi(\thetav)$ of interest. Given samples 
$(\thetav_{t})_{t=1}^T$ from pSGLD, we use the {\em sample average} $\hat{\phi}$
to approximate $\bar{\phi}$. In the analysis,
we define a functional $\psi$ that solves the following \emph{Poisson Equation}:
\begin{align}\label{eq:PoissonEq1}
	\mathcal{L} \psi(\thetav_{t}) =  \phi(\thetav_{t}) - \bar{\phi}~.
\end{align}
The solution functional $\psi(\thetav_{t})$ characterizes the difference between $\phi(\thetav_{t})$ 
and the posterior average $\bar{\phi}$ for every $\thetav_{t}$, thus would typically possess a unique 
solution, which is at least as smooth as $\phi$ under the elliptic or hypoelliptic settings \cite{MattinglyST:JNA10}. 
In the unbounded domain of $\thetav_{t}$, to make the presentation simple, we 
follow \cite{chen2015integrator} and make certain assumptions on the solution functional, $\psi$, of 
the Poisson equation \eqref{eq:PoissonEq1}, which are used in the detailed proofs. 

The mild assumptions of smoothness and boundedness made in the main paper are detailed as follows.

\begin{assumption}\label{ass:assumption1}
$\psi$ and its up to 3rd-order derivatives, $\mathcal{D}^k \psi$, are bounded by a
function $\mathcal{V}$, {\it i.e.}, 
$\|\mathcal{D}^k \psi\| \leq C_k\mathcal{V}^{p_k}$ for $k=(0, 1, 2, 3)$, $C_k, p_k > 0$. Furthermore, 
the expectation of $\mathcal{V}$ on $\{\thetav_{t}\}$ is bounded: $\sup_t \mathbb{E}\mathcal{V}^p(\thetav_{t}) < \infty$, 
and $\mathcal{V}$ is smooth such that 
$\sup_{s \in (0, 1)} \mathcal{V}^p\left(s\thetav + \left(1-s\right)Y\right) \leq C\left(\mathcal{V}^p\left(\thetav\right) + \mathcal{V}^p\left(Y\right)\right)$, $\forall \thetav, Y, p \leq \max\{2p_k\}$ for some $C > 0$.
\end{assumption}

\subsection{Proof of Theorem~1}

Based on Assumption~\ref{ass:assumption1},  we prove the main theorem.

\begin{proof}
First let us denote
\begin{align}
 	\tilde{\mathcal{L}}_t &= \left(G^{}(\thetav_t) \Big (\nabla_{\thetav}  \log p(\thetav_t)
+ \frac{N}{n}\sum_{i=1}^{n} \nabla_{\thetav}  \log p(\dv_{t_i} | \thetav_t ) \Big ) \right.
\nonumber \\
&+ \left.\Gamma(\thetav_t)\right) \cdot \nabla_{\thetav} 
+ \frac{1}{2}G^{}(\thetav)\left(G^{}(\thetav)^T\right):\nabla_{\thetav}  \nabla_{\thetav} ^T~,
\end{align}
the local generator of our proposed pSGLD with stochastic gradients, where $\av\cdot\bv\triangleq \av^\top\bv$ is the vector inner product, $\Amat:\Bmat\triangleq \mbox{tr}\{\Amat^\top\Bmat\}$ is the matrix double dot product. Furthermore,
let $\mathcal{L}$ be the true generator of the Langevin dynamic corresponding to
the pSGLD, {\it e.g.}, replacing the stochastic gradient in $\tilde{\mathcal{L}}_t$ with
the true gradient. As a result, we have the relation: 
\begin{align}
	\tilde{\mathcal{L}}_t = \mathcal{L} + \Delta V_t~,
\end{align}
where $\Delta V_t \triangleq (N \bar{g}(\thetav_t; \Dcal^t) - g(\thetav_t; \Dcal^t))^\top G^{} (\thetav_t) \nabla_{\thetav}$, $g(\thetav_t; \Dcal^t)$ is the full gradient, $\bar{g}(\thetav_t; \Dcal^t)$ is the stochatic gradient calculated from the $t$-th minibatch.

In pSGLD, we use the Euler integrator, which is a first order integrator. As a result, according to \cite{chen2015integrator},
for a test function $\phi$, we can decompose it as:
\begin{align} 
	\mathbb{E}[\psi(\thetav_{t})] &= e^{\epsilon_t\tilde{\mathcal{L}}_t} \psi(\thetav_{(t-1)}) + O(\epsilon_t^{2}) \nonumber\\
	&= \left(\mathbb{I} + \epsilon_t\tilde{\mathcal{L}}_t\right) \psi(\thetav_{(t-1)}) + O(\epsilon_t^{2})~,
\end{align}
where $\mathbb{I}$ is the identity map, \ie$\mathbb{I} f(x) = f(x)$.

According to the assumptions, there exists a functional $\psi$ that solves the following Poisson Equation:
\begin{align}\label{eq:PoissonEq}
	\mathcal{L} \psi(\thetav_{t}) =  \phi(\thetav_{t}) - \bar{\phi}~,
\end{align}
where $\bar{\phi}$ is defined in the main text.

Sum over $t = 1, \cdots, T$ in the above equation, take expectation on both sides, and use the Poisson 
Equation~\eqref{eq:PoissonEq} and the relation $\tilde{\mathcal{T}}_t = \Lcal + \Delta V_t$ to expand the 
first order term. We obtain
\begin{align}
	\sum_{t=1}^T\mathbb{E}\left(\psi(\thetav_{t})\right) &= \sum_{t=1}^T \psi(\thetav_{(t-1)}) + \sum_{t=1}^T \epsilon_t\mathcal{T}_t\psi(\thetav_{(t-1)}) 
	\nonumber\\
	& + \sum_{t=1}^T \epsilon_t\Delta V_l\psi(\thetav_{(t-1)}) + C\sum_{t=1}^T \epsilon_t^{2}~.
\end{align}
Divide both sides by $S_T$, we have
{\small
\begin{align}
	\hat{\phi} - \bar{\phi} &= \frac{\mathbb{E}\psi(\thetav_{t}) - \psi(\thetav_0)}{S_T} + \frac{1}{S_T}\sum_{l=1}^{T-1}\left(\mathbb{E}\psi(\thetav_{(t-1)}) + \psi(\thetav_{(t-1)})\right) 
	 \nonumber \\
	&+ \sum_{t=1}^T \frac{\epsilon_t}{S_T} \Delta V_t\psi(\thetav_{(t-1)}) + C\frac{\sum_{t=1}^T \epsilon_t^2}{S_T}~.
\end{align}
}

As a result, there exists some positive constant $C$, such that:
\begin{align}
	&\left(\hat{\phi} - \bar{\phi}\right)^2 \leq C\left(\frac{1}{S_T^2}\underbrace{\left(\psi(\thetav_0) - \mathbb{E}\psi(\thetav_{T})\right)^2}_{A_1}\right. \nonumber\\
	&+ \underbrace{\frac{1}{S_T^2}\sum_{t=1}^T\left(\mathbb{E}\psi(\thetav_{(t-1)}) - \psi(\thetav_{(t-1)})\right)^2}_{A_2}
	+ \sum_{t=1}^T \frac{\epsilon_t^2}{S_T^2}\left\|\Delta V_t\right\|^2 \nonumber\\
	&\left. + \left(\frac{\sum_{t=1}^T \epsilon_t^2}{S_T}\right)^2\right) \label{eq:mse1_w}
\end{align}

$A_1$ can be bounded by assumptions, and $A_2$ can be easily shown to be bounded by $O(\sqrt{\epsilon_t})$ due
to the Gaussian noise. It turns out that the resulting terms have order higher than those from the other terms, thus can be ignored
in the expression below. After some simplifications, \eqref{eq:mse1_w} is bounded by:
{\small
\begin{align}
	&\mathbb{E}\left(\hat{\phi} - \bar{\phi}\right)^2 \lesssim \sum_t \frac{\epsilon_t^2}{S_T^2}\mathbb{E}\left\|\Delta V_t\right\|^2 + \frac{1}{S_T} + \frac{1}{S_T^2} + \left(\frac{\sum_{t=1}^L \epsilon_t^{2}}{S_T}\right)^2 \nonumber\\
	&= C\left(\sum_t \frac{\epsilon_t^2}{S_T^2}\mathbb{E}\left\|\Delta V_t\right\|^2 + \frac{1}{S_T} + \frac{(\sum_{t=1}^T \epsilon_t^{2})^2}{S_T^2} \right)
\end{align}
}
for some $C > 0$. It is easy to show under the assumptions, all the terms in the above bound approach
zero. This completes the first part of the theorem. 
\hfill $\blacksquare$
\end{proof}

\section{B.$\quad$The proof for Corollary~2}

To prove Corollary~2, we first show the following results.

\begin{lemma}\label{lem:bound}
	Assume that the 1st-order and 2nd-order gradient are bounded, 
         then there exists some constant $M$, for k-th component of $\Gamma(\thetav_t)$, we have
	\begin{align}
		\left|\sum_{t=1}^T \Gamma_k(\thetav_t)\right| \leq M T \frac{(1 - \alpha)}{\alpha^{\frac{3}{2}} }~.
	\end{align}
\end{lemma}
\begin{proof}

Since $\Gamma(\thetav) $
is a diagnal matrix, we focus on one of its elements thus omit the index $k$ in the following.

First, the iterative form of exponential moving average can be written as a function of the gradients at all the previous timesteps:
\beqs
V(\theta_t) = 
\alpha V(\theta_{t-1}) + (1-\alpha) \bar{g}^2(\theta_t) \\
= (1-\alpha) \sum_{i=1}^{t} \alpha^{t-i} \bar{g}^2(\theta_i) 
\eeqs
Based on this, for each component of $\Gamma(\thetav_t)$, we have
\begin{align}
&\left|\sum_{t=1}^T \Gamma(\theta_t)\right| 
=  \left| \sum_{t=1}^{T} (1- \alpha) V^{ -\frac{3}{2}}(\theta_t)  \bar{g}(\theta_t)  \frac{\partial \bar{g}(\theta_t) }{\partial \theta_t}  \right|  \\
&=   \left|  \sum_{t=1}^{T} (1- \alpha)  
\frac{ \bar{g}(\theta_t) }{ \big (  \alpha V(\theta_{t-1}) + (1-\alpha) \bar{g}^2(\theta_t)   \big  ) ^{\frac{3}{2}} }
\frac{\partial \bar{g}(\theta_t) }{\partial \theta_t}   \right|  \\
&\ll  \left|   \sum_{t=1}^{T} 
\frac{  (1- \alpha)    }{ \alpha^{\frac{3}{2}}    V^{\frac{3}{2}}(\theta_{t-1})    }
\frac{\partial \bar{g}(\theta_t) }{\partial \theta_t}   \right| 
\end{align}

With the assumption that  the 1st-order and 2nd-order gradient are bounded, we have $  \left |   V^{-\frac{3}{2}}(\theta_{t-1})   \frac{\partial  \bar{g}(\theta_t)  }{\partial \theta_t}   \right |  \le M$, where $M$ is a constant independent of   $\{ \epsilon_t \}$. Therefore, 
$\left|\sum_{t=1}^T \epsilon_t \Gamma(\theta_{t}) \right | \ll M T (1- \alpha)/ \alpha^{\frac{3}{2}} $.	
\hfill $\blacksquare$
\end{proof}

Based on Lemma~\ref{lem:bound}, we now proceed to the proof of Corollary~2.
\begin{proof}
By dropping the $\Gamma(\thetav_t)$ terms, we get a modified version of the local generator corresponding to
the SDE of the pSGLD, defined as 
\begin{align*}
	\tilde{\mathcal{L}}_t = \mathcal{L} + \Delta \tilde{V}_t~,
\end{align*}
where $\Delta \tilde{V}_t = \Delta V_t + \Gamma(\thetav_t) \cdot \nabla_{\thetav} $ with $\Delta V_t$ defined in the proof of
Theorem~1.

Following the proof of Theorem~1, we can derive the bound for $(\hat{\phi} - \bar{\phi})^2$, which is no more than
\eqref{eq:mse1_w} with an extra term as:
\begin{align}
	&\left(\hat{\phi} - \bar{\phi}\right)^2 \leq C\left(\frac{1}{S_T^2}\underbrace{\left(\psi(\thetav_0) - \mathbb{E}\psi(\thetav_{T})\right)^2}_{A_1}\right. \nonumber\\
	&+ \underbrace{\frac{1}{S_T^2}\sum_{t=1}^T\left(\mathbb{E}\psi(\thetav_{(t-1)}) - \psi(\thetav_{(t-1)})\right)^2}_{A_2}
	+ \sum_{t=1}^T \frac{\epsilon_t^2}{S_T^2}\left\|\Delta V_t\right\|^2 \nonumber\\
	&\left. + \underbrace{\left\|\sum_{t=1}^T\frac{\epsilon_t}{S_T}\Gamma(\thetav_t)\right\|^2}_{A_3} + \left(\frac{\sum_{t=1}^T \epsilon_t^2}{S_T}\right)^2\right) \label{eq:mse2_w}
\end{align}
We can further relax $A_3$ above as:
\begin{align}
	A_3 &\leq \left(\sum_{k}\left|\sum_{t=1}^T\frac{\epsilon_t}{S_T}\Gamma_k(\thetav_t)\right|\right)^2 
	\nonumber\\
	&\leq \left(\sum_{k}\frac{\epsilon_1}{T\epsilon_T}\left|\sum_{t=1}^T\Gamma_k(\thetav_t)\right|\right)^2 
	\nonumber\\
	&\leq O\left(\frac{ (1 - \alpha)^2 }{\alpha^3}\right)~,
\end{align}
where the last inequality follows by using the bound from Lemma~\ref{lem:bound}.
Taking expectation on both sides, we arrive at the MSE:
{\scriptsize
\begin{align}
	&\mathbb{E}\left(\hat{\phi} - \bar{\phi}\right)^2 \leq
	\nonumber \\ & \hspace{-0.5cm}C\left(\sum_t \frac{\epsilon_t^2}{S_T^2}\mathbb{E}\left\|\Delta V_t\right\|^2 + \frac{1}{S_T} + \frac{(\sum_{t=1}^T \epsilon_t^{2})^2}{S_T^2} +
	 \mathbb{E}\left\|\sum_{t=1}^T\frac{\epsilon_t}{S_T}\Gamma(\thetav_t)\right\|^2 \right) 
	\nonumber \\
	&\leq \mathcal{B}_{\text{mse}} + O\left(\frac{ (1 - \alpha)^2  }{\alpha^3}\right)~.
\end{align}
}
for some $C > 0$.
\hfill $\blacksquare$
\end{proof}

\section{C.$\quad$The proof for Corollary~3}

\begin{proof}
By thinning samples from the pSGLD, we obtain a sequence of subsamples $\{\thetav_{t_1}, \cdots, \thetav_{t_m}\}$
from the original samples $\{\thetav_1, \cdots, \thetav_n\}$ where $m \leq n$ and $(t_1, \cdots, t_m)$ is a subsequence
of $(1, 2, \cdots, n)$. Since we use the 1st-order Euler integrator, based on the definition in \cite{chen2015integrator}, 
we have for the original samples:
\begin{align}\label{eq:komogorov_op}
	\tilde{P}_lf(\thetav_l) \triangleq \mathbb{E}f(\thetav_l) = e^{\epsilon_l\tilde{\mathcal{L}}_l}f(\thetav_l) + O(\epsilon_l^2)~,
\end{align}
where $\tilde{P}_l$ denotes the Kolmogorov operator. Now for samples between $t_i$ and $t_j$, {\it i.e.}, 
$\{\thetav_{t_i}, \cdots, \thetav_{t_j}\}$, we have
\begin{align}\label{eq:composition}
	\tilde{P}_{t_j}f(\thetav_i) = \tilde{P}_{t_j} \circ \cdots \circ \tilde{P}_{t_i} f(\thetav_i)~,
\end{align}
where $A \circ B$ denotes the composition of the two operators $A$ and $B$, {\it i.e.}, $A$ is evaluated on the output
of B. Now substitute \eqref{eq:komogorov_op} into \eqref{eq:composition}, and use the Baker-Campbell-Hausdorff 
formula \cite{Bakhturin01} for commutators, we have
\begin{align}
	\tilde{P}_{t_j}f(\thetav_i) &= e^{\sum_{l = i}^j\epsilon_l\tilde{\mathcal{L}}_l} f(\thetav_i) + O(\sum_{l = i}^j\epsilon_l^2)  \nonumber \\
	&\leq e^{S_{ij}\tilde{\mathcal{L}}_{ij}} f(\thetav_i) + O(S_{ij}^2)~,
\end{align}
where $S_{ij} \triangleq \sum_{l=i}^j \epsilon_l, \tilde{\mathcal{L}}_{ij} \triangleq \sum_{l = i}^j\frac{\epsilon_l}{S_{ij}}\tilde{\mathcal{L}}_l$. 
This means by thinning the samples, going from $\thetav_i$ to $\thetav_j$
corresponds to a 1st-order local integrator with stepsize $S_{ij}$ and a modified generator of the corresponding SDE as
$\tilde{\mathcal{L}}_{ij}$, which is in the same form as the original generator $\mathcal{L}$.

By performing the same derivation with the new generator $\tilde{\mathcal{L}}_{ij}$, we obtain the same MSE as in Theorem~1 in the main text.
\hfill $\blacksquare$
\end{proof}

\section{D.$\quad$The proof for bias-variance tradeoff}

\subsection{Bias-variance decomposition}
\beqs
\text{Risk} &:& R = \E [ ( \bar{\phi}  - \hat{\phi} )^2]  =  B^2 + V 
\label{eq:risk}
\eeqs

\begin{proof}

\begin{align*}
R 
&= \E [ ( \bar{\phi}  - \hat{\phi} )^2]\\
&= 
\E [ ( \bar{\phi} - \bar{\phi}_{\eta} + \bar{\phi}_{\eta} - \hat{\phi} )^2]  \\
&=  \E [ ( \bar{\phi} - \bar{\phi}_{\eta})^2  + (\bar{\phi}_{\eta} - \hat{\phi} )^2 
+ 2 ( \bar{\phi} - \bar{\phi}_{\eta}) (\bar{\phi}_{\eta} - \hat{\phi} )    ]  \\
&=  \E [ ( \bar{\phi} - \bar{\phi}_{\eta})^2  ]  +  \E [ (\bar{\phi}_{\eta} - \hat{\phi} )^2  ] \\ \nonumber
& \qquad+ 2  \E [ ( \bar{\phi} - \bar{\phi}_{\eta}) (\bar{\phi}_{\eta} - \hat{\phi} )    ]  \\
&=   ( \bar{\phi} - \bar{\phi}_{\eta})^2   +  \E [ (\bar{\phi}_{\eta} - \hat{\phi})^2  ] \\ 
&= B^2 + V
\end{align*}
where
\beqs
\text{Bias} &:& B = \bar{\phi}_{\eta} - \bar{\phi} \label{eq:bias}
\\
\text{Variance} &:& V = \E [ (  \bar{\phi}_{\eta} -  \hat{\phi})^2 ] \label{eq:varance}
\eeqs
\hfill $\blacksquare$
\end{proof}

\subsection{Variance term in risk of estimator}

\beqs
\text{Variance} &:& V = \E [ (  \bar{\phi} -  \hat{\phi})^2 ]
\approx  \frac{   A_{} (0) }{ M } \label{eq:varance}
\eeqs

\begin{proof}
\beqs
V &=&\ \E [ (  \bar{\phi}_{\eta} -  \hat{\phi} )^2 ] \nonumber \\
& = & 
\E \Big[ \Big (  \bar{\phi}_{\eta} -  
\frac{1}{T} \sum_{i=1}^{T}\phi (\thetav_i)  \Big )^2\Big ] \\
& = & 
\frac{1}{T^2}\E \Big[ \sum_{i=1}^{T}  \sum_{j=1}^{T}
\Big (   \bar{\phi}_{\eta} -  \phi(\thetav_i)  \Big ) 
\Big (   \bar{\phi}_{\eta} -  \phi(\thetav_{j})  \Big ) 
\Big ] \nonumber\\
& = & 
\frac{1}{T^2} \sum_{i=1}^{T}  \sum_{j=1}^{T}
A( |i-j|)
\\
& = &  
\frac{1}{T^2}  \sum_{i=1}^{T}  \left(\sum_{t=-\infty}^{\infty}
A( |t|) - \sum_{|t| > 2T}A( |t|)\right)  \label{eq:ac1}\\
& \approx &  
\frac{1}{T^2}  \sum_{i=1}^{T}  \sum_{t=-\infty}^{\infty}
A( |t|)  \label{eq:ac2}\\
& = &  
\frac{1}{T} \Big(  A( 0 )+ 2 \sum_{t=1}^{\infty}
A( t )
\Big ) \\
& = &  
\frac{A( 0 ) }{T} \Big(  1 + 2 \sum_{t=1}^{\infty}
\frac{A( t ) }{ A( 0 )  } 
\Big ) 
\label{eq:varance_derivation}
\eeqs
where the term $\sum_{|t| > 2T}A( |t|)$ is omitted from \eqref{eq:ac1} to \eqref{eq:ac2},
which is usually small according to the property of autocovariance function.

We repeat some defintions from the main paper~\cite{gamerman2006markov}. 
\beqs
A(t) = 
\E [ (  \bar{\phi}_{\eta} -  \phi(\thetav_0) )
(  \bar{\phi}_{\eta} -  \phi(\thetav_t) )
]
\eeqs
is the \textit{autocovariance function}, manifesting how strong two samples with a time lag $t$ are correlated. Its normalized version 
\beqs
\text{ACF} &:& 
\gamma(t) = 	\frac{A( t ) }{ A( 0 )  }  
\eeqs
is called the \textit{autocorrelation function} (ACF). 

\beqs
\text{ACT} &:& 
\tau=  \frac{1}{2} +  \sum_{t=1}^{\infty} 	\gamma(t)   
\eeqs
is the integrated \textit{autocorrelation time} (ACT), which measures the interval between independent samples.

Note that \textit{effective sample size} (ESS) is defined as
\beqs
\text{ESS} &:& M = 
\frac{ T }{ 1 + 2 \sum_{t=1}^{\infty} 	\frac{A( t ) }{A( 0 )  }  } 
\label{eq:ess}
\eeqs

Plugin the definition into the derivation for variance, we have 
\beqs
V \approx \frac{A( 0 ) }{T} \Big(  1 + 2 \sum_{t=1}^{\infty}
\frac{A( t ) }{ A( 0 )  } \Big ) 
= \frac{A( 0 ) }{M} 
\eeqs
\hfill $\blacksquare$
\end{proof}

\section{E.$\quad$More results on simulation}
We demonstrate our pSGLD on a simple 2D gaussian example,
{\scriptsize	$\Ncal \Big( \left [ \begin{array}{c} 0 \\ 0\end{array}  \right ], 
	\left [ \begin{array}{cc} 0.16&0 \\ 0& a \end{array} \right ] 
	\Big)$}. 
The first 600 samples of both methods for  different $a$ and $\epsilon$ are shown in Fig.~1.

Comparing the results for different stepsize $\epsilon$ at the same $a$, it can be seen that pSGLD can adapt stepsizes acorrding to the manifold geometry of different dimensions. 

When $a$ is rescaled from $0.5$ to $2$, stepsize $\epsilon=0.1$ is appropriate for SGLD at $a=0.5$, but not a good choice at $a=2$, because the space is not fully explored. This also implies that even if the covariance matrix of a target distribution is mildly rescaled, we do not have to choose a new stepsize for pSGLD. Whilst, the stepsize of the standard SGLD needs to be fine-tuned in order to obtain decent samples.

\begin{figure} 
	\begin{tabular}{ c  c  }
			\hspace{-4mm}
		\begin{minipage}{4.2cm}
			\includegraphics[width=4.2cm]{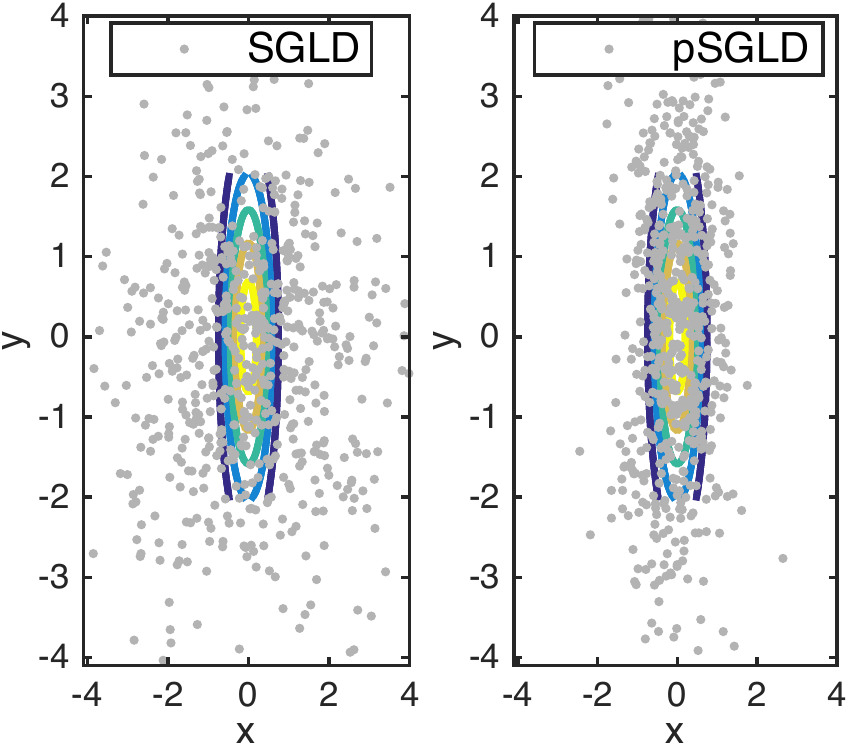} 
		\end{minipage}   &
			\hspace{-4mm}
		\begin{minipage}{4.2cm}\vspace{0mm}
			\includegraphics[width=4.2cm]{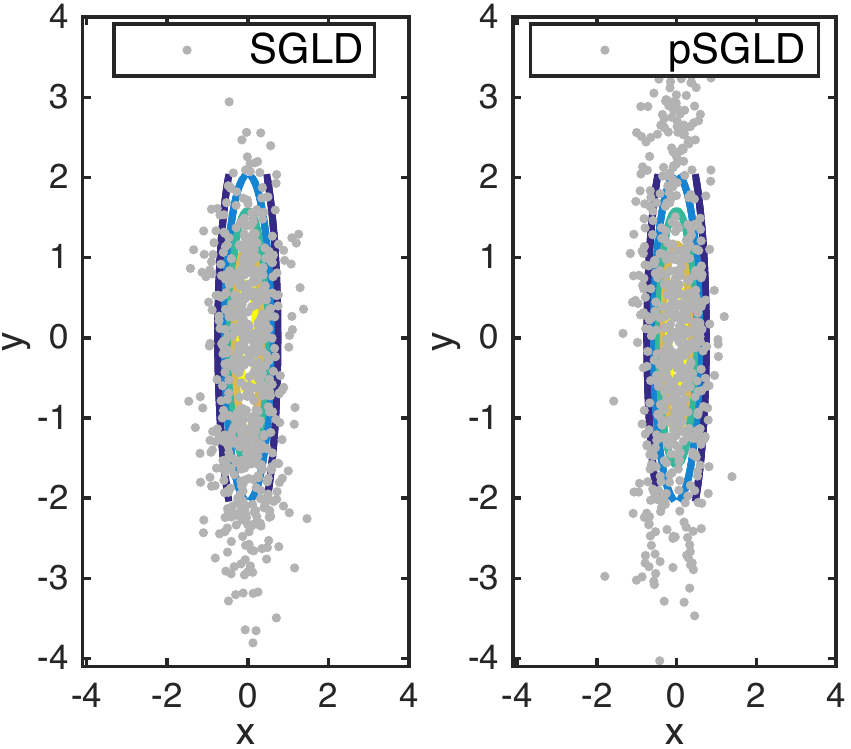} 
		\end{minipage} \\ 
		(a) $a = 2, \epsilon = 0.3$
		\vspace{2mm}
		&
		(b) $a = 2, \epsilon = 0.1$
		\\ 
			\hspace{-4mm}
		\begin{minipage}{4.2cm}
			\includegraphics[width=4.2cm]{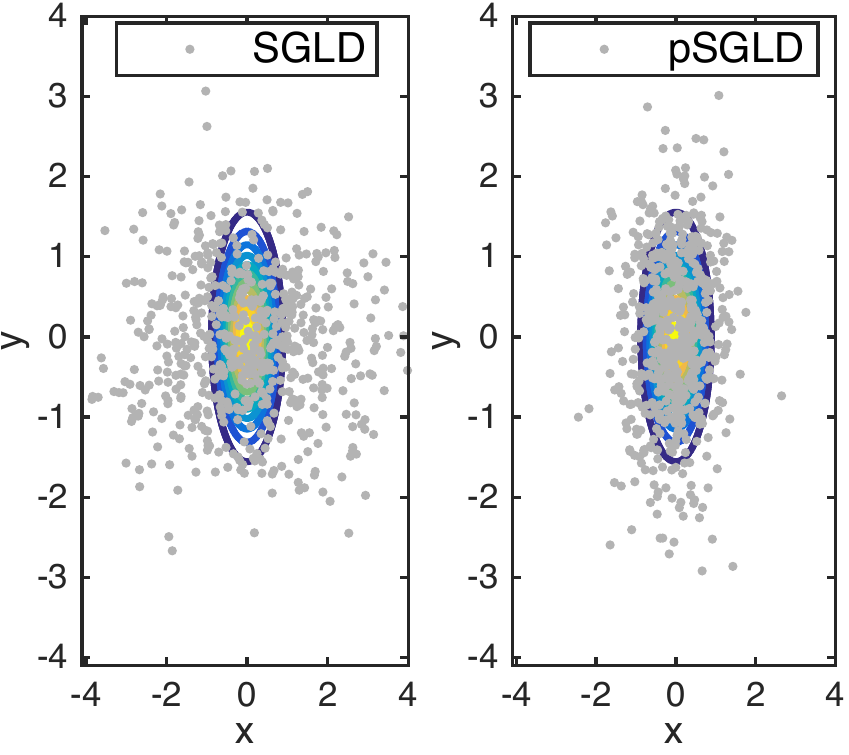} 
		\end{minipage}  &
			\hspace{-4mm}
		\begin{minipage}{4.2cm}\vspace{0mm}
			\includegraphics[width=4.2cm]{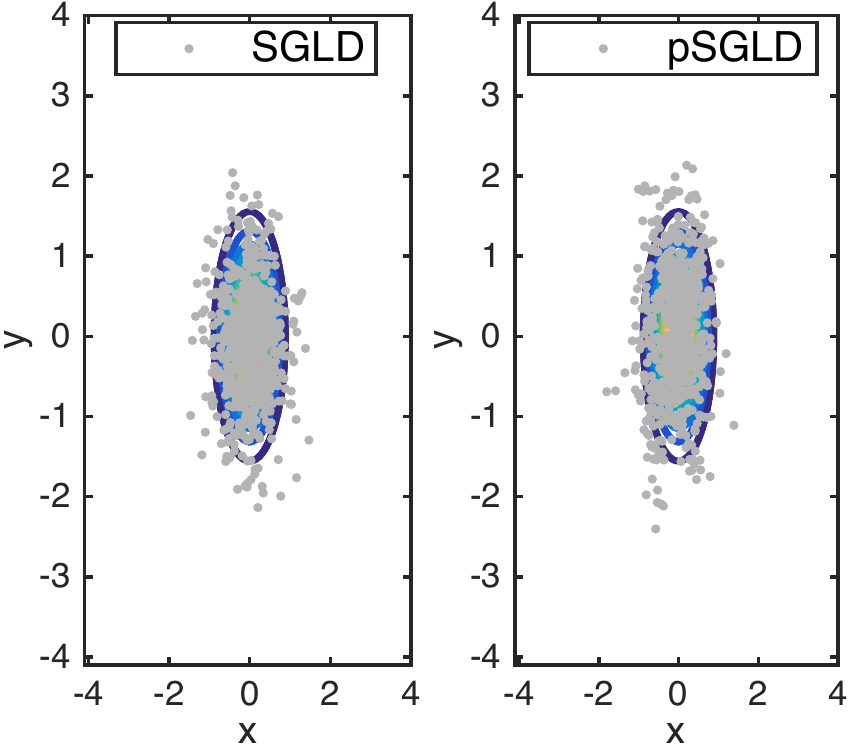} 
		\end{minipage} \\
		(c) $a = 0.5, \epsilon = 0.3$
		&
		(d) $a = 0.5, \epsilon = 0.1$
	\end{tabular}
	\label{fig:simulation}
	\caption{Simulation.}
\end{figure}		

\vspace{0mm}
\section{F.$\quad$More results on \\Feedforward Neural Networks}
Learning curves for network sizes of 400-400 and 800-800 on MNIST are provided in Fig.~\ref{fig:fnn} (a) and (b), respectively. Similar with results of network size 1200-1200 in the main paper, stochastic sampling methods take less iterations to converge, and the results are more stable than their optimization counterparts. Moreover, it can be seen that pSGLD consistently converges faster and better than SGLD and others.
\begin{figure}[t!] \centering
	
	\begin{tabular}{c c}
		\hspace{-4mm}
		\includegraphics[width=4.2cm]{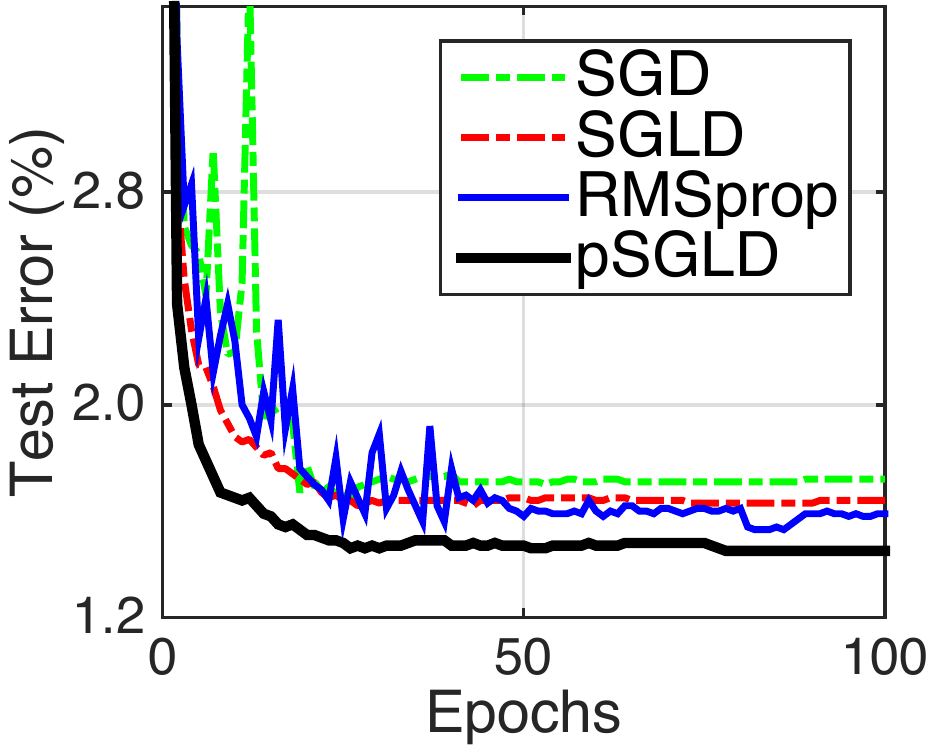} &
		\hspace{-4mm}
		\includegraphics[width=4.2cm]{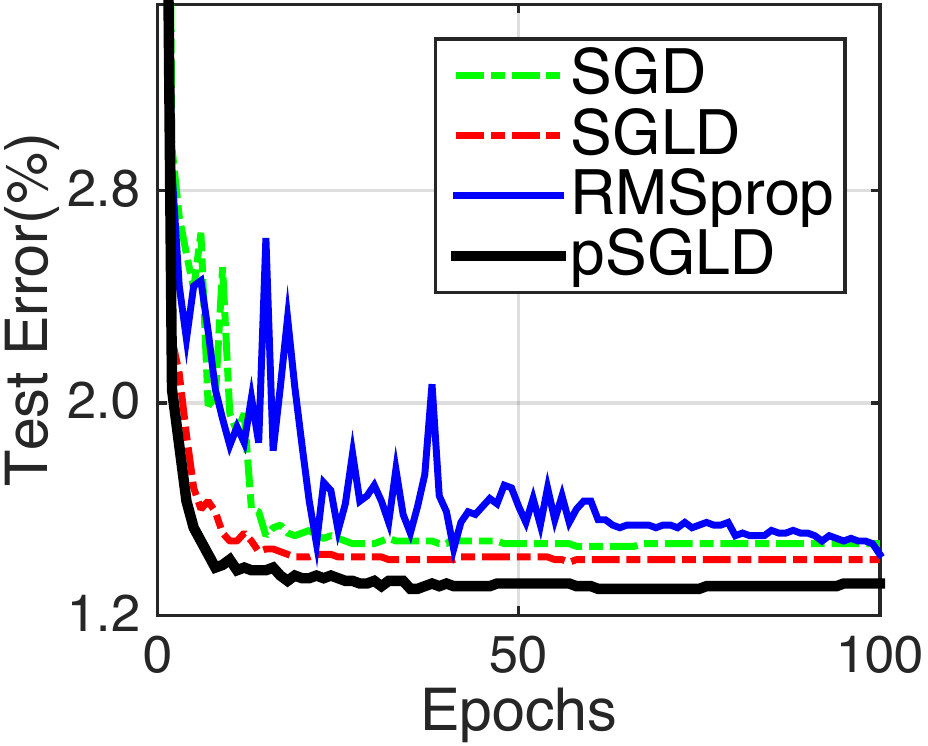} \\
		(a) 400-400 & (b) 800-800
	\end{tabular} \vspace{-1mm}
	\caption{Learning curves of FNN at different network sizes.}
	\label{fig:fnn}
	\vspace{0pt}
\end{figure}

\section{G.$\quad$More results on \\Convolutional Neural Networks} 
We use another fairly standard network configuration containing 2 convolutional layers on MNIST dataset. It is followed by a single fully-connected layer~\cite{lee2015deeply}, containing {\bf 500} hidden nodes that uses ReLU. Both convolutional layers use $5 \times 5$ filter size with 32 and 64 channels, respectively, $2 \times 2$ max pooling are used after each convolutional layer. 100 epochs are used, and $L$ is set to 20. The stepsizes for pSGLD and RSMprop are set to $\epsilon = \{1, 2\} \times 10^{-3} $ via grid search. 
For SGLD and SGD, this is $\epsilon =  \{1, 2\} \times 10^{-1} $. 

A comparison of test errors is shown in Table~1\ref{tab:cnn}, with the corresponding learning curves in Fig.~3\ref{fig:cnn}.  Again, under the same network architecture, CNN trained with traditional SGD gives an error of 0.81\%, while pSGLD has a significant improvement, with an error of 0.56\%. 

 \vspace{2mm}
\begin{minipage}{0.48\textwidth}\centering
	\hspace{-2mm} \vspace{2mm}
	\begin{minipage}[b]{0.42\textwidth}
		\centering
		\small
		\begin{tabular}{cc}
			\hline
			Method & Test error \\
			\hline
		pSGLD     & {\bf 0.56\%} \\
		SGLD				     &  0.76\%\\		  			
		RMSprop	    		   &  0.64\%\\
		SGD				          &  0.81\%\\	
			\hline
		\end{tabular}
		\vspace{5mm}
		\label{tab:cnn}
		\captionof{table}{Results of CNN.}
	\end{minipage}
	\hfill \hfill
	\begin{minipage}[b]{0.58\textwidth}
		\centering
		\includegraphics[width=4.45cm]{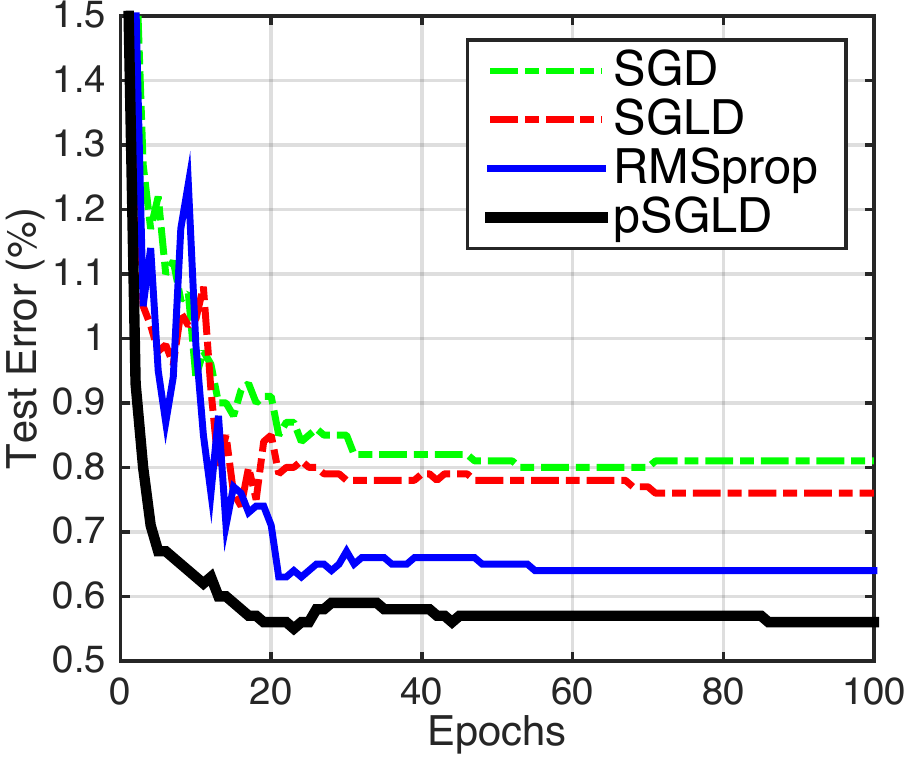}
		\label{fig:cnn}
		\vspace{-2mm}
		\captionof{figure}{Learning curves.}
	\end{minipage}
\end{minipage}

We also tested a similar 3-layer CNN with 32-32-64 channels on Cifar-10 RGB image dataset~\cite{krizhevsky2009learning}, which consists of $50,000$ samples for training and $10,000$ samples for testing. No data augmentation is employed for the dataset. We keep the same setting for pSGLD and SGLD from MNIST, and show the comparison on Cifar-10 in Fig.~\ref{fig:cnn_cifar}. pSGLD converges faster and reach a lower error.

\begin{figure}[h!]
	\centering	
	\begin{minipage}{8.0cm}	 \hspace{-8mm}\centering
		\includegraphics[width=5.3cm,height=4.0cm]{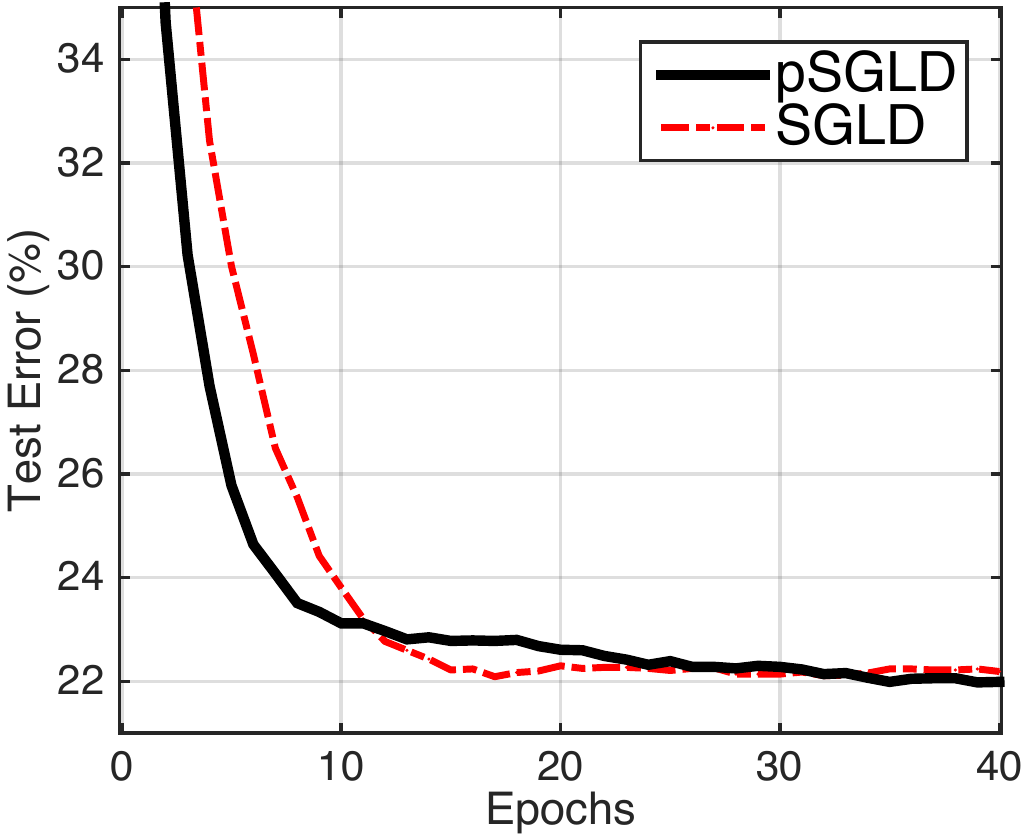} 
	\end{minipage}  
	\vspace{-3mm}
	\caption{Test learning curves of CNN on Cifar-10 dataset.}\label{fig:cnn_cifar}
	\vspace{-3mm}
\end{figure}

%